%% file: neurips_2026.tex
\documentclass{article}

 \usepackage[preprint]{neurips_2026}

\usepackage[utf8]{inputenc} % allow utf-8 input
\usepackage[T1]{fontenc}    % use 8-bit T1 fonts
\usepackage{hyperref}       % hyperlinks
\usepackage{url}            % simple URL typesetting
\usepackage{booktabs}       % professional-quality tables
\usepackage{amsfonts}       % blackboard math symbols
\usepackage{nicefrac}       % compact symbols for 1/2, etc.
\usepackage{microtype}      % microtypography
\usepackage{xcolor}         % colors
\usepackage{amsmath}
\usepackage{amsthm}
\usepackage{amssymb}
\newtheorem{theorem}{Theorem}[section]

\newtheorem{lemma}[theorem]{Lemma}
\newtheorem{corollary}[theorem]{Corollary}
\newtheorem{definition}[theorem]{Definition}
\newtheorem{assumption}[theorem]{Assumption}

\usepackage{algorithm}
\usepackage{algpseudocode}
\usepackage{adjustbox}
\usepackage{multirow}
\usepackage[table]{xcolor}
\usepackage{subcaption}
\usepackage{wrapfig}

\usepackage{xspace}
\newcommand{\mName}{{GRAIN}\xspace}

\title{GRAIN: Group Aggregation via Min-Norm Objective}

\author{%
  Nghia Bui \\
  Department of Data Science \\
  New Jersey Institute of Technology\\
  \texttt{ntb23@njit.edu} \\
  \And
  Jiarui Yao \\
  Boston Children's Hospital \\
  Harvard Medical School \\
  \texttt{jiarui.yao@childrens.harvard.edu} \\
  \AND
  Lijing Wang \\
  Department of Data Science \\
  New Jersey Institute of Technology\\
  \texttt{lijing.wang@njit.edu} \\
}

\begin{document}

\maketitle

\begin{abstract}
    Learning instability is a long-standing problem across machine learning, but it is especially acute in the overparameterized regime that defines modern deep learning: large models fine-tuned or trained on limited data traverse flat loss landscapes with many nearly-equivalent minima, and stochastic factors (initialization, data order, dropout, hardware non-determinism) can route optimization to very different solutions. The rise of large pretrained models (LPMs) makes the problem more urgent: training cost is high, downstream data is often small, and repeated runs for variance reduction are prohibitive. We introduce \textbf{GRAIN} (\textbf{G}roup \textbf{A}ggregation via m\textbf{IN}-norm objective), a lightweight training algorithm that replaces the mean aggregation used in mini-batch optimization (both across mini-batches and within a mini-batch) with a min-norm convex combination of group-wise gradients. \mName guarantees a non-negative inner product between the aggregated update and every group gradient, resolving intra- and inner-batch gradient conflict, and retains an $\mathcal{O}(1/T)$ convergence rate comparable to SGD. Under mild smoothness and absolute-continuity assumptions, the min-norm solution differs almost surely from the arithmetic mean, which yields a uniform-stability bound for \mName strictly tighter than the standard bound for SGD. Empirically across generation, classification, and regression at LPM scale, \mName delivers consistent improvements in mean performance and reductions in run-to-run variance over a broad suite of tasks, with no extra training-time or storage cost beyond a single backward pass.
\end{abstract}

\input{sections/introduction}
% \input{sections/related}
\input{sections/problem}
\input{sections/method}
\input{sections/experiment}

\input{sections/conclusion}

%%%%%%%%%%%%%%%%%%%%%%%%%%%%%%%%%%%%%%%%%%%%%%%%%%%%%%%%%%%%

\bibliography{ref}
\bibliographystyle{apalike}

%%%%%%%%%%%%%%%%%%%%%%%%%%%%%%%%%%%%%%%%%%%%%%%%%%%%%%%%%%%%

\appendix
\newpage
\input{sections/appendix}

%%%%%%%%%%%%%%%%%%%%%%%%%%%%%%%%%%%%%%%%%%%%%%%%%%%%%%%%%%%%

% \clearpage
% \input{checklist.tex}

\end{document}

%% file: sections/introduction.tex
\section{Introduction}
\label{sec:introduction}

Stochastic optimization is at the heart of modern machine learning, but it comes with a cost: two runs of the same algorithm with the same hyperparameters can produce noticeably different models. Weight initialization, data ordering, mini-batch sampling, dropout masks, and even hardware-level non-determinism all feed into this variance~\citep{picard2021torch,summers2021nondeterminism,madhyastha2019model,bethard2022we}. The phenomenon is ubiquitous: it appears in CNNs trained from scratch on imbalanced or noisy data, in regression with pretrained encoders fine-tuned on small targets, and most prominently in the fine-tuning of large pretrained models~\citep{dodge2020fine,mosbach2020stability,
zhang2020revisiting,bui2025assessing}. Figure~\ref{fig:variance} illustrates the magnitude of the problem: fixing all hyperparameters and varying only the seed produces over 50 percentage points (pp) gap between the best and worst seeds on PubMedQA, isolated catastrophic failures on SuperGLUE BoolQ, and smaller but persistent variance on noisy-label image classification using non-LPM models.

% \begin{figure*}[t]
%     \centering
%     \begin{subfigure}{0.3\linewidth}
%         \centering
%         \includegraphics[width=\linewidth]{figures/pubmed.pdf}
%         \caption{PubmedQA}
%     \end{subfigure}
%     \centering
%     \begin{subfigure}{0.3\linewidth}
%         \centering
%         \includegraphics[width=\linewidth]{figures/noisycifar.pdf}
%         \caption{Noisy CIFAR-100}
%     \end{subfigure}
%     \centering
%     \begin{subfigure}{0.3\linewidth}
%         \centering
%         \includegraphics[width=\linewidth]{figures/boolq.pdf}
%         \caption{SuperGLUE: BoolQ}
%     \end{subfigure}
%     \caption{Performance variances across multiple task while using the same training configurations.\textcolor{red}{TODO: improve the figure.}}
% \end{figure*}

\paragraph{Why instability is more severe in the overparameterized regime.} Classical learning theory studies the underparameterized regime, where the training loss has a unique (or nearly unique) minimizer and the randomness of SGD averages away. Modern deep learning operates in a fundamentally different regime: models are massively overparameterized with respect to the amount of available training data, the loss landscape admits a continuum of near-zero-loss solutions~\citep{zhang2021understanding}, and small perturbations in the optimization trajectory can route training toward very different minima with very different generalization behavior~\citep{mosbach2020stability,dodge2020fine}. Overparameterization is not the cause of randomness, but it is the \emph{amplifier} that turns everyday stochasticity into macroscopic test-accuracy swings.

\begin{wrapfigure}{r}{0.45\linewidth}
    \centering
    \vspace{-20pt}
    \includegraphics[width=\linewidth]{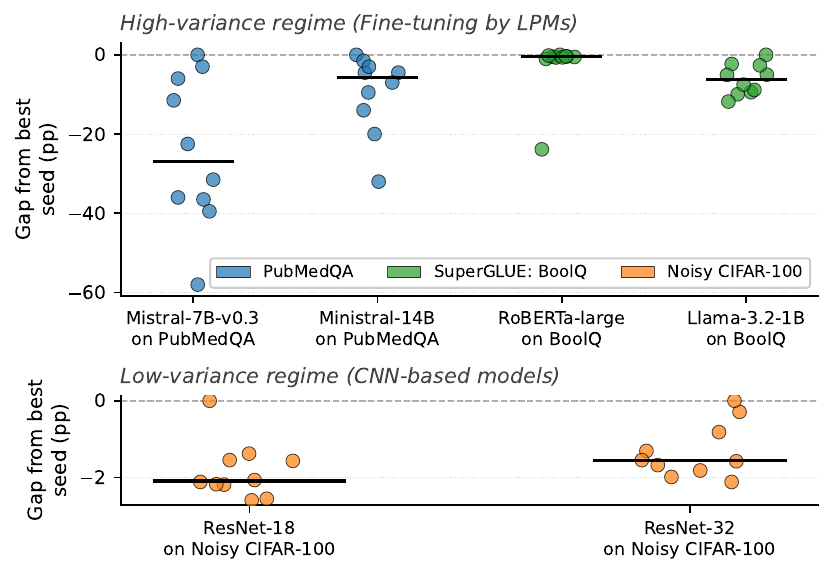}
    \caption{Seed-induced accuracy variance across six (model, task) configurations, each trained for 10 random seeds with identical hyperparameters. Each point is plotted as the gap below that configuration's best performance; black bars mark the median.}
    \label{fig:variance}
    \vspace{-15pt}
\end{wrapfigure}

\paragraph{Why the LPM era makes this urgent.}
Three properties of the current LPM era conspire to make instability
particularly costly:
\emph{(i)~Models are large}: A single fine-tuning run of a 7B or 14B model can take hours on multiple GPUs; repeating it many times to average away seed noise is often infeasible in practice.
\emph{(ii)~Downstream data is scarce}: Fine-tuning sets for domain-specific tasks are typically small, keeping the system deep in the overparameterized regime where instability is worst.
\emph{(iii)~Deployment stakes are high}: Models fine-tuned for medical, legal, or safety-critical applications are expected to behave consistently; run-to-run variance is a reliability liability, not just a research nuisance.
A training algorithm that reduces
seed-induced variance \emph{without} multiplying compute or storage
is therefore valuable well beyond a single domain.

\paragraph{The limits of existing remedies.}
\emph{Ensembles and weight averaging} (SnapshotEnsemble~\cite{wang2020wisdom,wang-etal-2023-two}, Model Soups~\cite{wortsman2022model}, Model
Stocks~\cite{jang2024model}, Stochastic Weight Averaging~\citep{izmailov2018averaging,gao2022revisiting,madhyastha2019model,nishida-etal-2025-instability}) are effective but scale training cost or storage with ensemble size, making them prohibitive at scale. 
\emph{Data-centric} remedies such as larger
datasets~\cite{dodge2020fine} or longer schedules~\cite{mosbach2020stability} are often impractical and still leave substantial variance. 
\emph{Noise injection} (LNSR~\citep{hua2021noise}, NoisyTune~\citep{wu2022noisytune}),
\emph{sharpness-aware methods} (SAM~\cite{foret2020sharpness}, ASAM~\cite{kwon2021asam},
Tilted-SAM~\cite{li2024tilted}), and \emph{regularization schemes} (label smoothing~\cite{muller2019does}, Mixout~\cite{lee2020mixout},
Dropout~\cite{srivastava2014dropout}, R-Drop~\cite{wu2021r}) all target generalization rather than the optimization process itself.
\emph{Gradient-conflict resolvers} from multi-task learning
(PCGrad~\citep{yu2020gradient}, CAGrad~\citep{liu2021conflict},
GradDrop~\citep{chen2020just}, MGDA~\citep{sener2018multi}) require a separate backward pass per ``task'' and detect conflict via cosine similarity, which fails when gradients are orthogonal or small in magnitude.
What is missing is a method that (i) \textit{directly targets the stochastic source of instability in the update rule itself}, (ii) \textit{adds negligible compute or storage}, and (iii) \textit{comes with a stability guarantee}.  

\textbf{Our approach.} We argue that the arithmetic mean used in SGD, both across mini-batches (inter-batch) and within a mini-batch (intra-batch), is precisely what exposes training to seed-induced instability: when gradients from different groups of examples conflict, their mean can be small or even zero, producing an ineffective update. We replace this mean with the \emph{min-norm convex combination} of group-wise gradients. Given a mini-batch partitioned into $m$ groups with gradients $g_1,\dots,g_m$, we solve $\lambda^\star = \arg\min_{\lambda\in\Delta^{m-1}} \|\sum_i \lambda_i g_i\|^2$ and take a step along $\bar g = \sum_i \lambda_i^\star g_i$. This update is cheap (a single backward pass plus a small quadratic program in $m$ variables) and, as we show, has several appealing properties:

\begin{itemize}
\item \textbf{Zero intra-group conflict.} By KKT conditions, $\bar g$
      has non-negative inner product with every $g_i$, so every
      group loss is guaranteed non-increasing to first order.
\item \textbf{Convergence.} Under standard smoothness, GRAIN
      converges to a stationary point $\mathcal{O}(1/T)$.
\item \textbf{Strictly tighter stability bound.} For smooth networks
      with continuous data density, the min-norm solution
      $\lambda^\star$ differs from the uniform weight
      $\tfrac{1}{m}\mathbf{1}$ almost surely; the strict inequality
      propagates through a Hardt--Recht--Singer-style
      argument~\citep{hardt2016train} and yields an instability bound
      strictly smaller than SGD's.
\item \textbf{No extra cost.} GRAIN requires only a single backward
      pass and a QP in the number of groups $m$ (typically 2--4); on
      multi-GPU LLM training the per-device gradients already
      provide a natural group partition, so the per-step overhead is
      a small QP solve relative to one forward/backward of an LPM.
\end{itemize}

% The closest prior work to ours is \emph{DSGD}~\citep{bui2026dsgd}, which downscales the gradients of correctly-classified examples within a mini-batch using a dynamic schedule. DSGD established that intra-batch gradient aggregation is a fruitful axis for stability, but its design has three structural limits: (i) it is classification-specific; (ii) it acts only intra-batch and only along the correctness axis, with a deterministic schedule rather than an adaptive aggregation rule; and (iii) its stability bound is non-strict. \mName generalizes the underlying intuition into a \textit{task-} and \textit{partition-agnostic} update rule that aggregates at both the intra- and inter-batch levels and admits a \emph{strictly} tighter instability bound than SGD.

\paragraph{Contributions.}
(i)~We introduce GRAIN, an optimizer-agnostic training procedure that
aggregates group-wise gradients via a min-norm convex combination,
replacing both inter-batch and intra-batch mean operations.
(ii)~We provide theoretical guarantees that apply to any smooth model
trained on data with a continuous density: zero group-gradient
conflict, $\mathcal{O}(1/T)$ convergence, and a strictly tighter
uniform-stability bound than SGD.
(iii)~We empirically validate GRAIN across generation, sequence and image
classification, and regression tasks at
LPM scale, reporting consistent
improvements in mean performance and run-to-run variance over a broad suite of baselines.

%% file: sections/problem.tex
% \section{From Seed Variance to Gradient Cancellation}
\section{Phenomenon}
\label{sec:phenomenon}

Before motivating \mName's design, we formalize the failure modes we
target, show that they arise across learning tasks (regression,
classification, generation), and identify a common underlying
mechanism: \textit{group-level gradient conflict}. 

\textbf{Diagnostic: group-level gradient conflict.} 
Detecting per-example gradient conflict in high-dimensional parameter
space is infeasible. We instead measure conflict at the \emph{group-level}, both \textbf{(1)} \emph{intra-batch} (between groups within a single mini-batch) and \textbf{(2)} \emph{inter-batch} (between consecutive mini-batches), by inspecting two quantities along the training trajectory: \textbf{aggregated gradient norm} and \textbf{between-group cosine similarity} of the group-wise gradients. 
A small net gradient combined with negative between-group
cosine similarity is the signature of \textit{gradient cancellation} (A phenomenon during optimization in which the expected gradient is approximately zero, or the losses of certain groups increase, preventing the model from fitting the training data and resulting in ineffective parameter updates.).

\begin{figure*}[t]   % full-width figure (note the *)
\centering
\includegraphics[width=\textwidth]{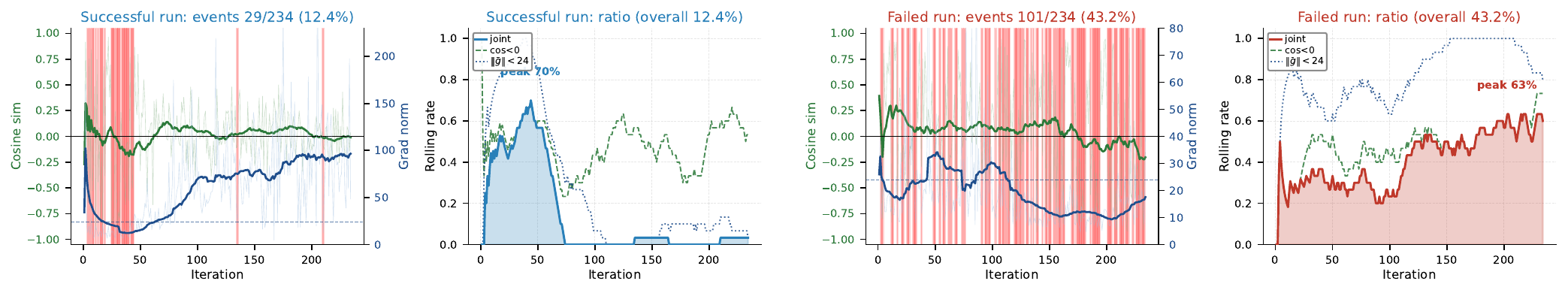}
\caption{Gradient cancellation across training (\texttt{RoBERTa-large} on RTE), where a cancellation event is an iteration with $\cos(g_i, g_j) < 0$ AND mean group-grad norm $\|\bar{g}\|=\tfrac{1}{2}(\|g_i\| + \|g_j\|)$ below the 25th percentile of the successful run's norms ($\approx 23.9$). Each run shows per-iteration cosine and norm with red-shaded events (panels~1,~3) and the rolling-30 fraction satisfying the joint condition, $\cos < 0$ alone, and $\|\bar{g}\| < 23.9$ alone (panels~2,~4). The failed run (accuracy 52.71\%) exhibits $43.2\%$ cancellation versus $12.4\%$ for the successful run (accuracy 84.47\%), with a sustained late-training rate above $50\%$.}
\label{fig:failed_suc_comparison}
\vspace{-12pt}
\end{figure*}

\textbf{Motivation.}
Figure~\ref{fig:failed_suc_comparison} makes the cancellation pattern concrete. We fine-tune \texttt{RoBERTa-large} on RTE for two runs differing only in random seed; within each mini-batch, examples are partitioned uniformly at random into two groups $g_i$ and $g_j$. Cancellation events ($\cos(g_i, g_j) < 0$ AND near-zero aggregated gradient norm) occupy $43.2\%$ of training iterations in the failed run versus $12.4\%$ in the successful one, with the failed run's rolling rate climbing past $50\%$ in late training while the successful run's decays after step $\sim 70$. Cancellation is therefore not a fixed property of the task or the model: the same model on the same data can swing between healthy training and chronic cancellation purely from initialization variance.
The mechanism generalizes beyond this single setting. Cancellation is specific neither to LPMs nor to any particular way of partitioning a mini-batch. Any grouping by class, by loss magnitude, by residual sign, or by arbitrarily index can produce group gradients that oppose one another, and the arithmetic mean used by SGD averages these opposing signals indiscriminately.
Similar evidence can be observed in the across-mini-batch gradient conflicts shown in Figure~\ref{global_level_grad} in Appendix~\ref{apdix:add_results}.
The effect is present in any model but becomes dominant in the \textit{overparameterized} regime, where many flat directions in the loss landscape allow small differences in the aggregated gradient to route training toward very different minima.
This motivates a principled alternative that is \textit{task-agnostic} and \textit{partition-agnostic}: rather than averaging group gradients, find the convex combination whose norm is small while remaining non-conflicting with every group. That is exactly what the min-norm objective delivers, and it is the core of \mName.

%% file: sections/method.tex
\section{Proposed Solution: \mName}
\label{sec:method}
 
Let $\mathcal{D}$ be a data distribution over $\mathcal{Z}=\mathcal{X}\times\mathcal{Y}$,
$f_\theta:\mathcal{X}\!\to\!\mathcal{Y}$ a model with parameters $\theta\in\Theta\subseteq\mathbb{R}^d$, and $\mathcal{L}(\theta,z)$ a differentiable per-example loss. At iteration $t$ we sample a mini-batch $\mathcal{B}_t$ of size $B$ i.i.d.\ from $\mathcal{D}$. Standard mini-batch SGD updates $\theta_{t+1}=\theta_t-\eta \cdot \tfrac{1}{B}\sum_{z_i\in\mathcal{B}_t}\nabla_\theta\mathcal{L}(\theta_t,z_i)$.

\paragraph{\mName update.}
We partition $\mathcal{B}_t$ into $m$ disjoint groups
$\mathcal{B}_t=\bigcup_{i=1}^m \mathcal{B}_{t,i}$, with
$\mathcal{B}_{t,i}\cap\mathcal{B}_{t,j}=\emptyset$ for $i\neq j$.
Let $g_i=\tfrac{1}{|\mathcal{B}_{t,i}|}\sum_{z\in\mathcal{B}_{t,i}}\nabla_\theta\mathcal{L}(\theta_t,z)$
be the group gradient. Define the convex hull of group gradients as:
\[
\mathcal{G} = \Big\{\, g\in\mathbb{R}^d \;\Big|\; g=\sum_{i=1}^m\alpha_i g_i,\;
\alpha\in\Delta^{m-1} \,\Big\},
\quad
\Delta^{m-1}=\Big\{\alpha\in\mathbb{R}_{\geq 0}^m : \sum_i\alpha_i=1\Big\}.
\]
SGD corresponds to $\alpha=\tfrac{1}{m}\mathbf{1}$. \mName instead solves
the \emph{min-norm} problem~\cite{sener2018multi}:
\begin{equation}
\label{eq:minnorm}
\alpha^\star_t \;=\; \arg\min_{\alpha\in\Delta^{m-1}}
\Big\|\sum_{i=1}^m \alpha_i g_i\Big\|^2,
\qquad
\bar g_t \;=\; \sum_{i=1}^m \alpha^\star_{t,i}\, g_i,
\end{equation}
and performs the update:
\begin{equation}
\label{eq:update}
\theta_{t+1} \;=\; \theta_t \;-\; \eta\, \bar g_t.
\end{equation}

The problem in \eqref{eq:minnorm} is a convex QP in $m$ variables; we
solve it via the Frank--Wolfe algorithm of \citet{sener2018multi}, which
converges in $\mathcal{O}(m)$ inner iterations. For the small $m$ we use
in practice ($m\in\{2,3,4\}$), solving \eqref{eq:minnorm} takes
microseconds relative to one forward/backward of an LPM.
 
\paragraph{Two levels of aggregation.}
The grouping can be applied both \emph{within} a mini-batch (\textbf{intra-batch})
and \emph{across} mini-batches (\textbf{inter-batch}). In the intra-batch
setting, groups are formed inside $\mathcal{B}_t$, for example by
splitting $\mathcal{B}_t$ into $m$ sub-batches (our default). In the
inter-batch setting, the ``groups'' are $k$ consecutive mini-batch
gradients $\{G_t, G_{t-1}, \ldots, G_{t-k+1}\}$, and \mName aggregates
them with a min-norm combination. Inter-batch \mName requires only the
gradients already computed in the previous $k-1$ steps; no additional
backward pass is needed. Section~\ref{subsec:ablation_and_sensitivity} ablates both choices.

\paragraph{Optimizer-agnostic.}
The min-norm aggregation is independent of the choice of
first-order optimizer. To use \mName with AdamW, Adafactor, or any
other adaptive method, replace the batch-mean gradient passed to
the optimizer with $\bar g_t$ from \eqref{eq:minnorm} and run the
optimizer's usual update; the only change in the training loop is
which scalar combination of group gradients is fed to the
optimizer. 
% We use AdamW in all our experiments unless noted otherwise.
 
% \textbf{Implementation innovation.} For LLM traning due to multiple GPUs useage, we treat gradient of individual GPU as a gradient according to 1 group then combine them together based on min-norm objective hence the computational cost is nearly identical to normal training-finetuning.
% \textcolor{red}{add your parallel training implementation here. }

\paragraph{Multi-GPU implementation.}
On distributed LLM training, \mName composes naturally with data parallelism: each GPU's local-shard gradient is one group $g_i$, an all-reduce gathers them on rank~0 which solves the min-norm QP and broadcasts $\alpha^\star_t$, and each device forms $\bar g_t$ as a weighted sum locally. The communication payload matches vanilla data-parallel training (gradient tensors plus one length-$m$ vector broadcast), and the QP solve is microseconds, so per-iteration wall-clock cost is effectively unchanged for $m$ up to the number of GPUs.
 
The full algorithm is given in Appendix~\ref{apx:algo}.
 
\section{Theoretical Analysis}
\label{sec:theory}
 
We state the four core theoretical properties of GRAIN. Throughout,
we use $\eta$ for the learning rate, $L_s$ for the smoothness
constant of $\mathcal{L}$ (the Lipschitz constant of $\nabla\mathcal{L}$), and $H$
for the Lipschitz constant of $\mathcal{L}$ in $\theta$ (so $\|\nabla\mathcal{L}\|\!\leq\!H$).
Proofs are deferred to Appendix~\ref{apx:theory_proofs}.

\begin{definition}[Instability]
\label{def:instability}
    Given a stochastic training algorithm $\mathcal{A}$ and a test point
    $z$, the (uniform) instability of $\mathcal{A}$ at $z$ is:
    \begin{equation}
        \operatorname{Ins}_{z,\mathcal{A}}
        = \mathbb{E}\big[\,|\mathcal{L}(\theta,z)-\mathcal{L}(\theta',z)|\,\big],
    \end{equation}
    where $\theta$ and $\theta'$ are parameters returned by two
    independent runs of $\mathcal{A}$ from different
    seeds (initializations and/or data orderings) and the expectation
    is over the seed randomness.
\end{definition}

% We study the \textit{instability} of a model trained under stochastic factors.
% \textbf{Proposed approach:} We divide a mini-batch into smaller mini-batches (groups) and using min-norm solution to weight gradients according to each individual group. More specifically, let $\mathcal{B}_t$ is the sampled mini-batch at $t$. We divide $\mathcal{B}_t$ into $m$ groups $\{B_1, B_2, \dots, B_m\}$ where $\forall i, j: B_i \cap B_j = \emptyset$ and $\cup_{i={1,\dots, m}} B_i = \mathcal{B}_t$. Let $g_i$ is the returned gradient according to $B_i$. 
% Let $G = \{g \in \mathbb{R}^n | g = \sum_{i=1}^m \alpha_i g_i; \forall i:\alpha_i \geq 0; \sum_{i=1}^m \alpha_i = 1\}$ be the set of all possible combinations of all individual gradients. 
% Easy to see that SGD algorithm combined gradient $\tilde{g} \in G$. 
% Let $\alpha^* = \text{argmin}_\alpha \|g\|$, \mName's update gradient would be $g^* = \sum_{i=1}^m\alpha^* g_i$.

% \subsection{Convergence Analysis}

\begin{theorem}[Zero conflict]
\label{thm:zero_conflict}
    Let $\bar g$ be the \mName update defined in \eqref{eq:minnorm}. Then for every group gradient $g_i$, $\langle g_i,\bar g\rangle \geq \|\bar g\|^2 \geq 0$. In particular, taking a step $-\eta\bar g$ does not increase any group loss.
    % Applying \mName update rule resolve gradient conflict across group-losses. i.e. $\langle g_i, \bar{g} \rangle \geq 0$. Hence all individual group losses decreased.
\end{theorem}

Theorem~\ref{thm:zero_conflict} follows directly from the KKT conditions of the min-norm problem and distinguishes \mName from
PCGrad-style projections which only zero conflict after detection via negative cosine similarity. 
% and DSGD-style scheduled rescaling (which depends on a hand-tuned schedule). 

% \begin{theorem}[Convergence]
% \label{thm:convergence}
%     Assume that the loss function $\mathcal{L}$ is L-smooth, under \mName's update rule $\mathcal{L}$ will converged to a stationary point after a finite update steps.
% \end{theorem}

\begin{theorem}[Convergence]
\label{thm:convergence}
Assume each per-example loss is $L_s$-smooth and bounded below by
$\mathcal{L}^\star$, and $0<\eta<1/L_s$. Let $\{\theta_t\}$ be generated by
the GRAIN update~\eqref{eq:update}. Then:
\begin{equation}
\min_{0\le t<T}\|\bar g_t\|^2
\;\leq\;
\frac{2(\mathcal{L}(\theta_0)-\mathcal{L}^\star)}{(\eta-L_s\eta^2/2)\,T}
\;=\; \mathcal{O}(1/T),
\end{equation}
so $\bar g_t\to 0$. When all groups are i.i.d.\ from $\mathcal{D}$, this
implies $\nabla\mathcal{L}(\theta_t)\to 0$, i.e.\ $\theta_t$ converges to a
stationary point with a finite update steps. 
\end{theorem}
 
The strictness in our stability bound rests on a key technical fact: under mild assumptions (smooth activations, continuous data density, and a generic-position assumption on the parametric family), the min-norm solution $\alpha^\star$ differs almost surely from the uniform weight $\tfrac{1}{m}\mathbf{1}$, so $\|\bar g\| < \|\tfrac{1}{m}\sum_i g_i\|$ almost surely. The full statement and proof are deferred to Theorem~\ref{thm:minnorm-neq-mean} in Appendix~\ref{apx:proof-minnorm}.

\begin{theorem}[Strictly tighter instability bounds]
\label{thm:bounds}
    Assume $\mathcal{L}$ is $H$-Lipschitz in $\theta$ and the conditions of
    Theorem~\ref{thm:minnorm-neq-mean}  (Appendix~\ref{apx:proof-minnorm}) hold. 
    The instability upper bound of \mName is strictly lower than that of SGD while using the same learning rate schedule $\{\eta_t\}_{t=0}^{T-1}$, more specifically:
    \begin{equation}
    \label{eq:bound}
    \begin{split}
        \operatorname{Ins}_{z\sim \mathcal{D}, \text{\mName}} &< 2H^2 \sum_{i=0}^{T-1} \eta_t + \| \theta_0-\theta'_0 \|, ~~~
        \operatorname{Ins}_{z\sim \mathcal{D}, \text{SGD}} \leq 2H^2 \sum_{i=0}^{T-1} \eta_t + \| \theta_0-\theta'_0 \|,
    \end{split}
    \end{equation}
    where \mName and SGD are examined from the same starting points of two different initialization $\theta_0$ and $\theta_0'$ from different seeds.
\end{theorem}

Theorem~\ref{thm:bounds} bounds an expected pointwise loss difference and certifies  a smaller worst-case-allowed instability for \mName, not a smaller realized empirical metric variance reported in Section~\ref{sec:experiment}.

\begin{definition}[Sharpness]
\label{def:sharpness}
    The $\rho$-sharpness of the loss function $\mathcal{L}$ at $\theta$ is:
    \begin{equation}
        S_{\rho}(\theta) := \text{max}_{\| \epsilon \| \leq \rho} \mathcal{L}(\theta + \epsilon) - \mathcal{L}(\theta).
    \end{equation}
\end{definition}

\begin{theorem}[Loss sharpness bound]
\label{thrm:sharpness_bound}
    GRAIN loss function can be rewrite as a weighted aggregation of all group losses: ${\mathcal{L}}(\theta) = \sum_{i=1}^m \alpha^\star l_i(\theta)$.
    Assuming that each loss function of individual group $\{l_i\}_{i=1}^m$ is $L$-smooth,
    % and (2) $\rho \ll \frac{2 \| \bar{g} \|}{\|H\|}$ where $\bar{g} = \nabla \bar{\mathcal{L}}(\theta)$ and $H$ is the Hessian. 
    GRAIN loss sharpness at $\theta$ is bounded:
    \begin{equation}
        S_{\rho}^{GRAIN}(\theta) := \text{max}_{\| \epsilon \| \leq \rho}{\mathcal{L}}(\theta + \epsilon) - {\mathcal{L}}(\theta) \leq \rho \| \bar{g} \| + \frac{\rho^2L}{2}.
    \end{equation}
\end{theorem}

If assumptions in Theorem~\ref{thm:minnorm-neq-mean} hold i.e. we have $ \|\bar{g}\| < \| g \| $ hence the sharpness bound for \mName strictly tighter than SGD ones.

%% file: sections/experiment.tex
\section{Empirical Experiments}
\label{sec:experiment}

\textbf{Baselines.} 
We compare \mName against (1) fully fine-tuning (\textbf{FFT}) or LoRA~\cite{hu2022lora} (LLMs above 1B) fine-tuning (\textbf{LoRA});
% using the training guidelines from~\cite{mosbach2020stability, zhang2020revisiting};
% Or FT with LoRA fine-tuning (\textbf{FT-LoRA}) for LLMs above 1B size. 
(2) FFT/LoRA with focal loss~\cite{lin2017focal} (\textbf{FocalLoss});
\textit{noise injection approaches (for pretrained models)}: 
(3) \textbf{NoisyTune}~\cite{wu2022noisytune}; \textit{gradient conflict resolver}: we apply (4) \textbf{PCGrad}~\cite{yu2020gradient} on gradients according to each individual group using the random same grouping strategy;
% proposed to resolve the gradient conflicts between tasks in multitask learning; 
% \textit{ensemble methods}: (6) bagging ensemble of $N$ single learners based on majority voting (\textbf{ENS ($\times N$)}); 
(5) stochastic weight averaging (\textbf{SWA})~\cite{izmailov2018averaging};
(6) \textbf{SAM}~\cite{foret2020sharpness}.
We report mean performance (\textit{Mean}) and standard deviation (\textit{STD}) over 10 arbitrarily chosen random seeds. Full hyperparameters, dataset statistics, and baseline implementations are in Appendix~\ref{apx:exp}. 
Detailed grouping settings of \mName are listed in Table~\ref{tab:grain_settings}.

\subsection{Generative Task}
\label{subsec:generative}

\textbf{Models:} Qwen family: \texttt{Qwen2-7B-base}~\cite{yang2024qwen2},
\texttt{Qwen2.5-14B-base}~\cite{yang2024qwen25} and Mistral family:
\texttt{Mistral-7B-v0.3}~\cite{jiang2023mistral7b},
\texttt{Ministral-3-14B-Base-2512}~\cite{liu2026ministral3}.
\textbf{Datasets:} GSM8K~\cite{cobbe2021training}, PubmedQA~\cite{jin2019pubmedqa} (we use the labeled partition with training:testing ratio of 0.8:0.2).
\textbf{Metrics:} We evaluate generative tasks using zero shot evaluation and adopt Accuracy (exact match) as the measure. 

\textbf{Results.} 
Table~\ref{tab:gen} reports the results of generative tasks. \mName is \textbf{the only method that avoids collapsed failure on every (task, model) pair} ($^*$-free across both tasks; LoRA, NoisyTune, and SWA each show $^*$ on at least two PubMedQA cells, including the dramatic LoRA failure on \texttt{Mistral-7B} at $46.05_{\pm 19.00}$). On PubMedQA, where seed instability is most severe due to the small labeled split, \mName attains the best Mean on \texttt{Qwen-7B} ($74.15$, $+1.55$ over LoRA) and \texttt{Qwen-14B} ($76.58$, $+2.00$ over LoRA), and stabilizes the catastrophic \texttt{Mistral-7B} run from $46.05_{\pm 19.00}$ to $67.60_{\pm 0.61}$. On the more stable GSM8K benchmark, where every method runs without collapse, \mName achieves the best or tied-best Mean on every model with the lowest STD across the board, showing that \mName preserves accuracy and tightens variance even when the baseline is already stable.

\input{tables/generative}

\subsection{Classification Task}
\label{subsec:classification}

\textbf{Sequential classification.}
\textbf{Models:}
We adopt widely used models for sequence classification, covering two architectural paradigms: the encoder-only \texttt{RoBERTa-Large}~\cite{liu2019roberta} and the decoder-only \texttt{LLaMA-3.2-1B}~\cite{touvron2023llama}.

\textbf{Datasets:} BoolQ, RTE from SuperGLUE~\cite{wang2019superglue}; and MRPC from GLUE~\cite{wang2018glue}.

\textbf{Image classification.} We examine the training stability of \mName and baseline methods under \textbf{train-time distribution shift}, where the training data is perturbed by randomness. Models are then evaluated on the same clean, balanced test set, so performance variance reflects the algorithm's stability rather than evaluation noise. In this study, we consider two types of distribution shift: class imbalance~\cite{cui2019class, azizzadenesheli2019regularized} and label noise~\cite{jiang2020beyond} (used for low-variance regime analysis in Section~\ref{subsec:low_variance}).

\textbf{Models:} We use \texttt{ResNet-18} and \texttt{ResNet-32}~\cite{he2016deep} to evaluate \mName under label noise. For class imbalance experiments, we adopt \texttt{ViT-base-patch16-224}~\cite{dosovitskiy2021an} as the backbone.

\textbf{Datasets:} Following~\cite{cao2019learning}, we created imbalanced CIFAR-100 with two imbalance ratios (100:1 and 50:1) and two imbalance patterns: long-tailed (exponential decay)~\cite{cui2019class} and step imbalance~\cite{azizzadenesheli2019regularized}, resulting in four distinct datasets. With noisy label settings~\cite{jiang2020beyond} we randomly flip the label of a data point with ratios of $\{0.2, 0.4, 0.6, 0.8\}$.

\textbf{Results.}
Table~\ref{tab:sequence} reports sequence classification results; Table~\ref{tab:image_cls} reports image classification results. The same pattern holds in both: \mName is \textbf{the only method without a single collapsed failure}, while every baseline produces at least three starred entries. On RoBERTa-large, FFT and noise-injection baselines collapse on every sequence task ($\text{STD} \approx 5\text{--}16$); \mName cuts STD by an order of magnitude (e.g., RTE: $82.59_{\pm 1.24}$ vs. FFT $76.17_{\pm 15.44}$) and posts the best clean Mean on BoolQ ($85.04_{\pm 0.38}$). On \texttt{Llama-3.2-1B}, FocalLoss exhibits universal training failures (every entry starred), while \mName is the only method that beats FFT on BoolQ ($84.32$ vs. $73.98$) with substantially improved stability. On imbalanced CIFAR (Table~\ref{tab:image_cls}), \mName achieves the best mean accuracy on all eight (dataset $\times$ imbalance type $\times$ ratio) cells with consistently the lowest or near-lowest STD; the gain is largest on the more challenging step-imbalance settings, where SAM and FocalLoss exhibit super high STDs (e.g., FocalLoss $\text{STD}=14.22$ on CIFAR-10 ST-100, SAM $\text{STD}=11.55$ on CIFAR-10 LT-50) while \mName stays under $1.0$ throughout.

\input{tables/sequence}

\input{tables/vision}

\subsection{Regression Task}
\label{subsec:regression}

\textbf{Models:}
\texttt{Roberta-Large},
\texttt{Llama-3.2-1B}.
\textbf{Datasets:} STS-B from GLUE~\cite{wang2018glue}.  \textbf{Metrics:} Pearson correlation. 

\textbf{Results.} Table~\ref{tab:sequence} shows that \mName matches the performance of FFT ($91.56\%$ versus $91.86\%$ and $90.78\%$ versus $90.12\%$ across the two backbone models) while outperforming all other baselines in semantic similarity regression task. 
We also evaluate \mName on a simple regression task, namely Diabetes~\cite{pedregosa2011scikit}, using a two-layer MLP with 32 neurons per layer. \mName achieves a lower root mean squared error of $54.04 \scriptstyle \pm 0.96$, outperforming standard SGD training, which obtains $57.01 \scriptstyle \pm 0.78$.
Across both setups, \mName consistently outperforms existing training methods.

\subsection{Overall Performance Across Task Categories}
\label{subsec:overall_performance}

% \begin{wrapfigure}{r}{0.55\linewidth}
%     \centering
%     \vspace{-12pt}
%     \includegraphics[width=\linewidth]{figures/grain_heatmaps.pdf}
%     \caption{Per-method, per-task-category summary across all
%     (model, dataset, seed) configurations. 
%     \textbf{(a)} Collapsed failure ratio.
%     \textbf{(b)} Best-performance ratio.
%     % \textbf{(a)} Collapsed failure ratio: fraction of (model $\times$ dataset) configurations where the method produced at least one collapsed run across 10 seeds.
%     % \textbf{(b)} Best-performance ratio: fraction of configurations where the method achieved the highest Mean among baselines, with collapsed runs excluded from the comparison. \mName is the only method with zero failures across all populated categories.
%     }
%     \label{fig:summary_heatmap}
%     \vspace{-8pt}
% \end{wrapfigure}

To complement the per-task tables, we summarize each method's behavior across task categories in Figure~\ref{fig:summary_heatmap}. The figure reports two ratios computed over all (model, dataset, seed) configurations: the \emph{collapsed failure ratio} (the fraction of (model $\times$ dataset) configurations where the method produced at least one collapsed run across 10 seeds) and the \emph{best-performance ratio} (the fraction of configurations where it achieved the highest Mean among baselines, with collapsed runs excluded from the comparison).
% \textcolor{red}{TODO: after neurips submission, continue with GRAIN m and k tuning. e.g., set k=2, then increase m until GRAIN surpasses LoRA or all baselines.}

\textbf{\mName is the only method that never collapses.}
Every baseline collapses on a substantial fraction of sequence-classification configurations (50--67\%), and several also collapse on generative tasks. \mName, in contrast, shows a uniformly white row in panel~(a): zero failures across all populated configurations. This is the empirical signature of Theorem~\ref{thm:bounds}'s strictly tighter uniform-stability bound, which forbids the worst tail of run trajectories that produces the collapsed runs in baselines.

\textbf{Robustness does not come at the cost of accuracy.} 
Panel~(b) confirms that avoiding collapse does not require trading off peak performance. \mName achieves the highest Mean in 75\% of generative configurations and 50\% of sequence-classification configurations, substantially more than any baseline (next-best: FFT/LoRA at 33\% on sequence classification, no baseline exceeds 12\% on generative tasks). On image classification, \mName attains the highest Mean in every populated cell ($100\%$).
We attribute this consistency to two complementary mechanisms grounded in our theory. On LPM tasks where seed instability produces bimodal outcomes, Theorem~\ref{thm:bounds} tightens the expected generalization-gap bound via the stability-to-generalization argument~\citep{hardt2016train}: while this does not directly imply improved means, in the overparameterized regime training loss is near-zero across healthy runs, so test performance is dominated by whether a run avoids the collapsed left tail; the tighter bound forbids the worst trajectories and removing this tail raises the mean. \mName's lead is largest where baselines collapse (PubMedQA on \texttt{Mistral-7B}, \texttt{RoBERTa-large} sequence classification, long-tailed CIFAR-100) and tied or near-tied on stable settings (GSM8K, balanced CIFAR-10). On image classification under distribution shift, however, the gradient field itself contains persistent intra-batch conflict between majority and minority classes, and gains arise even where no baseline collapses (e.g., \mName $60.17$ vs.\ SAM $58.50$ on CIFAR-100 ST-100); Theorem~\ref{thm:zero_conflict} provides the more direct account, since the min-norm aggregation explicitly resolves this structural conflict at every step rather than acting through tail removal.

\textbf{Why baselines fail unevenly across task categories.} The asymmetry across panels traces to two interacting factors: the loss landscape and each baseline's cancellation-handling mechanism. On generative fine-tuning, LoRA itself collapses in only 2 of 8 configurations indicating that the trivial solution is less attractive at token-level granularity with larger training sets, so most baselines also avoid collapse and merely lose accuracy. On sequence classification with small-data tasks (RTE, MRPC, BoolQ on \texttt{RoBERTa-large}), FFT collapses because the loss landscape admits a clean trivial-solution basin, and every baseline collapses with it. Across tasks, each one fails for its own reason: PCGrad triggers only on strongly negative cosine, missing the orthogonal or small-magnitude conflicts that also produce collapse; SAM's worst-case perturbation inherits the same cancellation structure as the unperturbed gradient, and the trivial basin is itself flat, so SAM can reinforce rather than escape it; FocalLoss suppresses high-confidence examples and shrinks the effective batch as confidence rises; NoisyTune injects uniform randomness that fails to stabilize the sensitive directions responsible for divergence; SWA averages nearby checkpoints that remain trapped in the same collapsed basin. The pattern repeats on imbalanced CIFAR: FFT shows no collapses, but baselines fail unevenly across the eight configurations, revealing that \textit{distribution shift surfaces baseline-specific weaknesses that a clean training distribution conceals}. \mName intervenes at both intra- and inter-batch levels via a min-norm aggregation that does not depend on detecting any specific conflict signature, which is why it stays consistent across all task categories where every other baseline fails on at least one.

\subsection{Ablation Study and Sensitivity Analysis}
\label{subsec:ablation_and_sensitivity}

% We examine the contribution of each level of applying Min-Norm solver in GRAIN: \textbf{(1)} The global level applying min-norm objective in consecutive mini-batches; \textbf{(2)} Applying min-norm objective on the gradients according to individual group that 1 mini-batch is divided into. 
% Moreover, we also examine the effectiveness and robustness of \mName w.r.t the \textbf{(3)} the number of consecutive mini-batches and \textbf{(4)} the number of groups we divide the mini-batch into . More specifically we use \texttt{Qwen2-7B} as our backbone model and we evaluate \mName's variants on PubmedQA.

We analyze the contribution of each level at which the Min-Norm solver is applied in \mName: \textbf{(1)} \textit{a global-level} (inter-batch), where the min-norm objective is applied across consecutive mini-batches, and \textbf{(2)} \textit{a local-level} (intra-batch), where the min-norm objective is applied to gradients computed over groups obtained by partitioning each mini-batch.
In addition, we study the effectiveness and robustness of \mName with respect to \textbf{(3)} the number of consecutive mini-batches ($k$) and \textbf{(4)} the number of groups used to partition each mini-batch ($m$). Concretely, we use \texttt{Qwen2-7B} as the backbone model and evaluate variants of \mName on PubMedQA. The results are shown in Figure~\ref{fig:sensitivity}.

\begin{figure*}[t]
    \centering
    \begin{minipage}{0.64\textwidth}
        \centering
        \includegraphics[height=3cm, keepaspectratio]{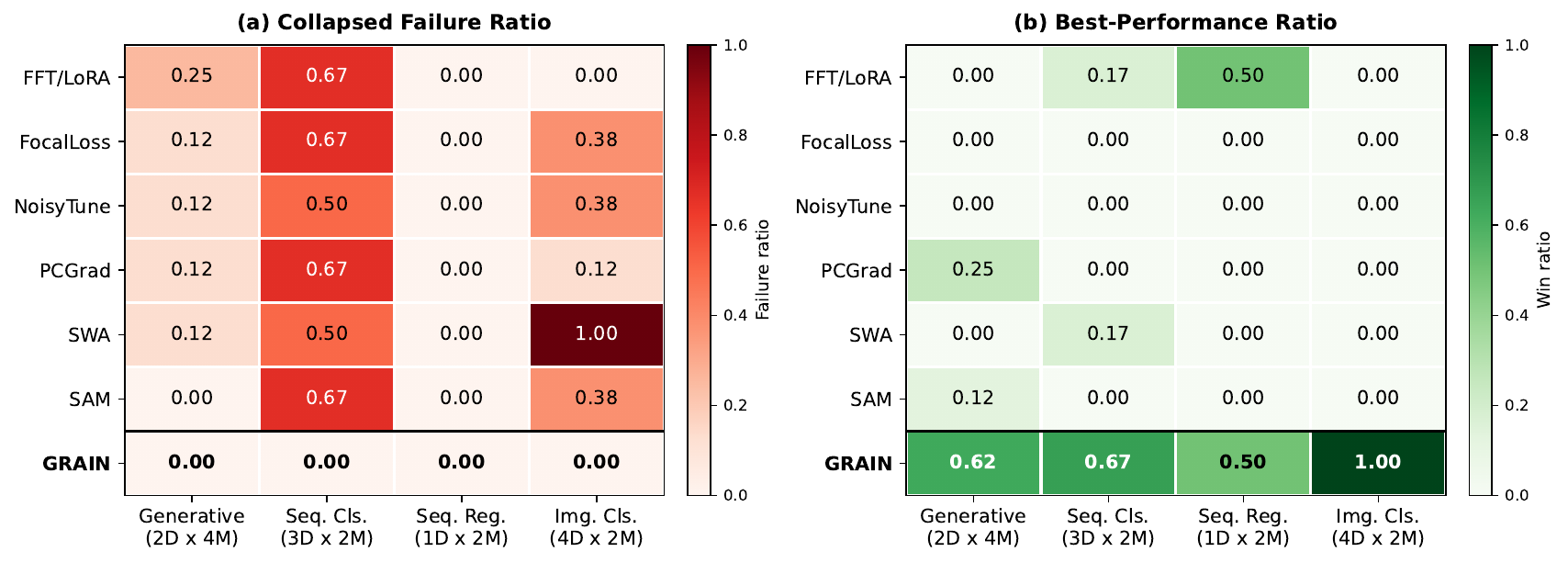}
        \caption{Per-method, per-task-category summary across all
        (model, dataset, seed) configurations. 
        \textbf{(a)} Collapsed failure ratio.
        \textbf{(b)} Best-performance ratio. 
        % \textcolor{red}{There are errors in the heatmap, e.g, in (a) NoisyTune on Generative is not 0. }
        }
        \label{fig:summary_heatmap}
    \end{minipage}
    \vspace{-8pt}
    \hfill
    \begin{minipage}{0.33\textwidth}
        \centering
        \includegraphics[height=3cm, keepaspectratio]{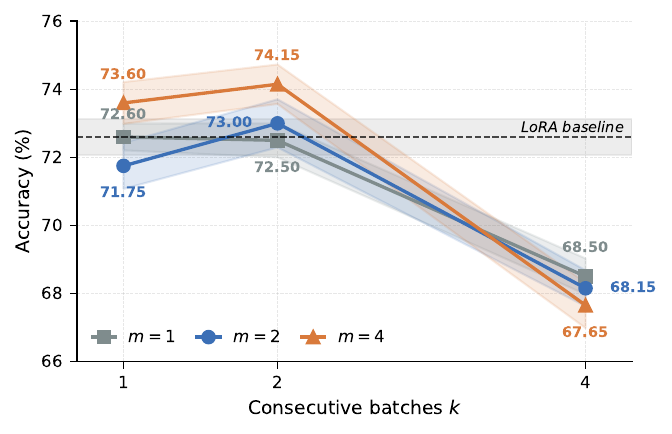}
        \caption{Sensitivity analysis on PubmedQA using Qwen2-7B as the backbone model.}
        \label{fig:sensitivity}
    \end{minipage}
    \vspace{-8pt}
\end{figure*}

\textbf{Ablation.} The LoRA baseline reaches $72.60 \pm 0.52$ on PubMedQA, which corresponds to ($m=1, k=1$) in our factorial. Each level alone provides limited benefit. Pure global aggregation ($m=1, k=2$) yields $72.50$, statistically indistinguishable from baseline. Pure local aggregation ($k=1$) is mixed: $m=2$ underperforms baseline at $71.75$, while $m=4$ surpasses it at $73.60$, indicating that finer partitioning is required for the local-level to contribute on its own. The combination is what produces the largest improvement: ($m=4, k=2$) reaches $74.15$, the best cell observed and the only configuration that clearly beats baseline. The two levels are therefore complementary rather than redundant. The local-level solver attacks intra-batch gradient conflict, while the global-level solver damps cross-batch trajectory noise that the group-level cannot see; neither mechanism alone fully captures the variance \mName reduces, but their composition does.

\textbf{Sensitivity.} \textit{Number of groups $m$.} The theory motivates ``the larger the better'' for $m$: in the extreme case where each group contains a single example, the min-norm solver resolves every individual gradient conflict within a mini-batch, recovering the strongest form of the no-conflict guarantee in Theorem~\ref{thm:zero_conflict}. The result is consistent with this at $k \leq 2$, where accuracy is monotone in $m$ ($m=1 < m=2 < m=4$ at $k=2$). The ordering inverts at $k=4$, but only because all three $m$ values break at that setting; the inversion reflects a failure of the global-level rather than a problem with finer partitioning. \textit{Number of consecutive mini-batches $k$.} Performance improves from $k=1$ to $k=2$ for the partitioned settings ($m \in \{2, 4\}$) but degrades sharply at $k=4$ across all $m$, losing roughly $4$--$5$ points relative to baseline. The story for $k$ is asymmetric. \emph{Recent} cross-batch gradient conflicts are genuinely harmful and worth resolving, since the current update is directly impacted by the previous one; \emph{distant} cross-batch gradients are largely irrelevant, since the parameters have already moved through several updates and the relationship resembles a short-memory Markov chain in which only the recent past carries information about the present. With a large $k$ the min-norm solver spends its weight reconciling stale information rather than suppressing genuine conflict.

\textbf{Practical takeaway.} Set $m$ as large as compute allows and $k=2$. With our multi-GPU implementation $m$ scales naturally up to the number of devices at no extra wall-clock cost, so the configuration that maximizes $m$ subject to the GPU budget is the one we recommend.

% \textcolor{red}{TODO: keep own reminder, the actual m, k settings are not consistent with the findings in this sensitivity study. To further demonstrate this, we should run $m = 1, 2, 4, 8, ..., batch_size$; $k = 1, 2, 4, 8$ on several model+task pairs. }

\subsection{Low-Variance Regime: Non-LPM Models}
\label{subsec:low_variance}

\input{tables/noisy_vision}
% \input{tables/resnet_imbalance}
% \begin{wrapfigure}{r}{0.42\linewidth}
%     \centering
%     \vspace{-14pt}
%     \includegraphics[width=\linewidth]{figures/grain_low_variance_gains.pdf}
%     \caption{Low-variance regime: per-configuration gains from \mName on non-LPM models. Each point is one (model, dataset) configuration.
%     The x-axis shows accuracy improvement; the y-axis shows variance reduction. The shaded green region marks improvement on both axes. \textcolor{red}{Include extra results here.}}
%     \label{fig:low_variance_gains}
%     \vspace{-8pt}
% \end{wrapfigure}

While the most dramatic seed effects appear in fine-tuning of large pretrained models, seed-induced variance is also present when training non-LPM models from scratch or under noisy supervision, although the variance is quite small in magnitude. To verify that \mName's benefits extend to this regime, we evaluate it on noisy-label and imbalanced image classification with \texttt{ResNet-18} and \texttt{ResNet-32} on CIFAR-10 and CIFAR-100 across four label-noise ratios $\{0.2, 0.4, 0.6, 0.8\}$ detailed results are shown in Table~\ref{resnet:noisylabel_imbalance}. Across most tasks, \mName achieves higher accuracy with lower variance. 
\mName improves accuracy on 29 of 32 configurations and lower variance on 22 of 32. Even in this regime, where each configuration's seed-to-seed spread is on the order of $1$ percentage point, \mName's group gradient modulation produces consistently improvements, indicating that the gradient-cancellation phenomenon addressed by \mName is not unique to LPM fine-tuning, but rather a general characteristic of mini-batch training.

\subsection{Empirical Evidence on Loss Landscape Sharpness}

We provide empirical evidence supporting Theorem~\ref{thrm:sharpness_bound}. Following the experimental setup in~\cite{li2018visualizing}, we train \texttt{ResNet-32} and \texttt{ResNet-56} without residual connections (denoted as \texttt{Resnet-32\_nr} and \texttt{Resnet-56\_nr} for short) on CIFAR, both with and without \mName, and visualize their loss landscapes. Figure~\ref{fig:loss_landscape} visualizes the loss landscapes with and without \mName. It is observable that \texttt{Resnet-56\_nr} trained without \mName exhibit sharper loss landscapes, whereas \mName leads to noticeably smoother optimization surfaces. \texttt{Resnet-32\_nr} shows similar pattern which is omitted for the sake of brevity.

\begin{wraptable}{r}{0.45\textwidth}
    \centering
    \begin{adjustbox}{max width=0.4\textwidth}
        \begin{tabular}{lcc}
        \toprule
                                   & \textbf{CIFAR-10} & \textbf{CIFAR-100} \\
        \midrule
        \texttt{Resnet-32\_nr}     &  ${9.27}$  &  ${9.37}$ \\
        \rowcolor{gray!20} + \mName&  $\mathbf{9.03}$  &  $\mathbf{8.65}$ \\
        \texttt{Resnet-56\_nr}     &  ${13.31}$ &  ${13.27}$ \\
        \rowcolor{gray!20} + \mName&  $\mathbf{11.86}$ &  $\mathbf{13.11}$ \\
        \bottomrule
        \end{tabular}
    \end{adjustbox}
    \caption{Error rate on CIFAR(\%) using \texttt{Resnet-56\_nr} and \texttt{Resnet-32\_nr} as backbone models.}
    \label{tab:no_res}
\end{wraptable}

As illustrated in Figure~\ref{fig:loss_landscape}, \mName smooths the loss landscape of \texttt{ResNet-56\_nr} and reduces its error rate from 13.31\% to 11.86\%. 
Table~\ref{tab:no_res} presents error rate(\%) on CIFAR-10 and CIFAR-100 by \texttt{ResNet-32\_nr} and \texttt{ResNet-56\_nr}. We observe consistent improvements across tasks and models, with error rate decreases significantly, e.g., 0.72\% reduction (9.37\% vs 8.64\%) on CIFAR-100 by \texttt{ResNet-32\_nr} and 1.45\% reduction (13.31\% vs 11.86\%) on CIFAR-10 by \texttt{ResNet-56\_nr}. 
% Although training original \texttt{Resnet-56} achieves a substantially lower error rate (5.89\%), \mName still provides a significant performance improvement when applied to \texttt{ResNet-32\_nr} and \texttt{ResNet-56\_nr}. \textcolor{red}{I don't understand the connection in the last sentence. why original REsnet-56 training error rate matters to GRAIN  training on ResNet-32\_nr and ResNet-56\_nr?? where does 5.89\% come from?}

\begin{figure*}[t]
    \centering
    \begin{subfigure}{0.4\linewidth}
        \centering
        \includegraphics[width=\linewidth]{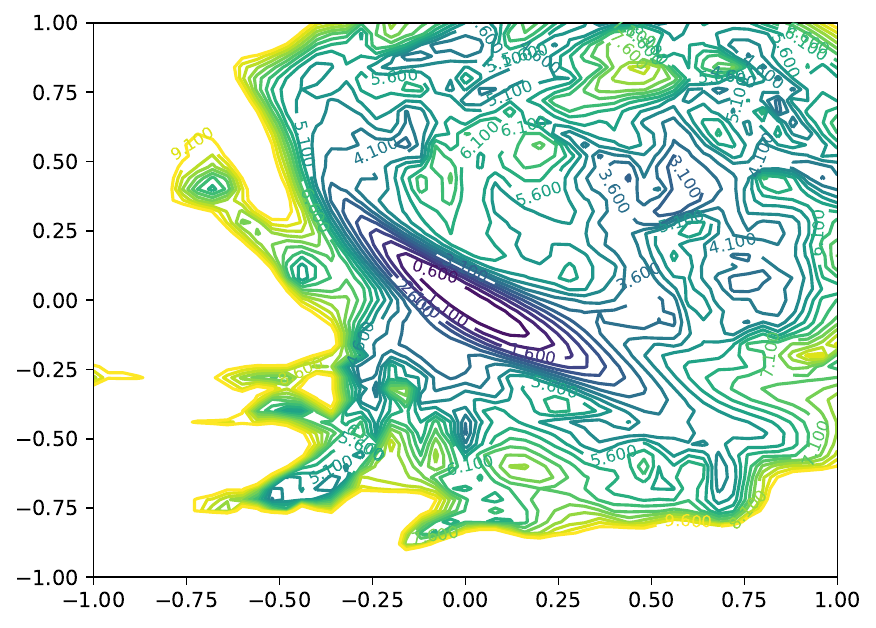}
        \caption{\textbf{without} \mName.}
    \end{subfigure}
    \hspace{10pt}
    \begin{subfigure}{0.4\linewidth}
        \centering
        \includegraphics[width=\linewidth]{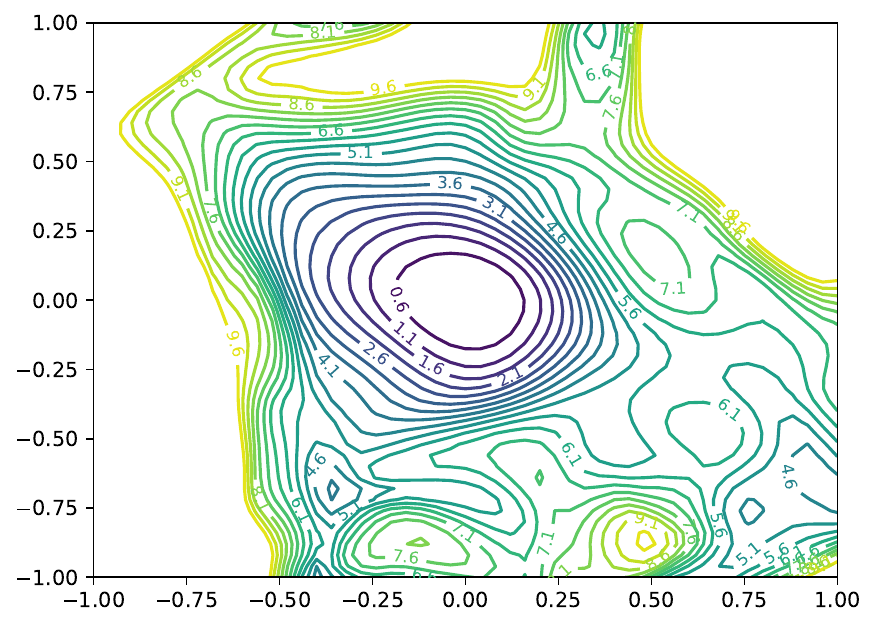}
        \caption{\textbf{with} \mName.}
    \end{subfigure}
    \caption{\texttt{Resnet-56\_nr} on CIFAR-10 loss landscapes with and without \mName, figures are visualized following~\cite{li2018visualizing} using the provided checkpoint.}
    \label{fig:loss_landscape}
\end{figure*}

%% file: tables/generative.tex
\begin{table}[t]
\centering
\begin{adjustbox}{max width=\textwidth}
\begin{tabular}{l|cccc cccc}
\toprule
\multirow{2}{*}{Method} & \multicolumn{4}{c}{\textbf{PubMedQA}} & \multicolumn{4}{c}{\textbf{GSM8K}}\\
\cmidrule(lr){2-5}\cmidrule(lr){6-9}
 & \texttt{Qwen-7B} & \texttt{Qwen-14B} & \texttt{Mistral-7B} & \texttt{Ministral-14B} & \texttt{Qwen-7B} & \texttt{Qwen-14B} & \texttt{Mistral-7B} & \texttt{Ministral-14B} \\
\midrule
LoRA          & \underline{$72.60{\scriptstyle \pm 0.52}$} & $74.58{\scriptstyle \pm 2.08}$ & $^*46.05{\scriptstyle \pm 19.00}$ & $^*67.90{\scriptstyle \pm 9.94}$ & $72.87{\scriptstyle \pm 0.71}$ & $82.94{\scriptstyle \pm 0.35}$ & $47.95{\scriptstyle \pm 0.74}$ & $79.24{\scriptstyle \pm 0.50}$ \\
 + FocalLoss& $71.62{\scriptstyle \pm 0.62}$ & $73.80{\scriptstyle \pm 1.03}$& $^*20.60{\scriptstyle \pm 8.40}$& $74.66{\scriptstyle \pm 0.76}$& $73.98{\scriptstyle \pm 0.65}$& $82.18{\scriptstyle \pm 0.65}$& $47.53{\scriptstyle \pm 0.75}$&$76.42{\scriptstyle \pm 1.51}$\\
+ NoisyTune     & $70.05{\scriptstyle \pm 0.55}$ & $74.20{\scriptstyle \pm 0.67}$ & $^*64.86{\scriptstyle \pm 7.44}$ & \underline{$76.70{\scriptstyle \pm 0.82}$} & $74.43{\scriptstyle \pm 0.45}$ & $84.46{\scriptstyle \pm 0.42}$ & $48.20{\scriptstyle \pm 0.78}$ & $78.77{\scriptstyle \pm 0.96}$ \\
+ SWA           & $68.15{\scriptstyle \pm 1.25}$ & $73.65{\scriptstyle \pm 0.85}$ & $^*67.55{\scriptstyle \pm 5.76}$ & $75.95{\scriptstyle \pm 0.37}$ & $73.90{\scriptstyle \pm 0.29}$ & $85.45{\scriptstyle \pm 0.51}$ & ${48.65{\scriptstyle \pm 0.88}}$& $78.77{\scriptstyle \pm 0.72}$ \\
+ SAM           & $70.95{\scriptstyle \pm 1.09}$ & $75.65{\scriptstyle \pm 1.20}$ & $66.95{\scriptstyle \pm 2.71}$ & $\mathbf{76.80{\scriptstyle \pm 1.32}}$ & $72.88{\scriptstyle \pm 0.69}$ & $85.45{\scriptstyle \pm 0.93}$ & $46.49{\scriptstyle \pm 0.89}$ & \underline{$79.63{\scriptstyle \pm 0.80}$} \\
+ PCGrad        & $66.40{\scriptstyle \pm 0.81}$ & \underline{$76.10{\scriptstyle \pm 0.52}$} & $\mathbf{69.90{\scriptstyle \pm 1.10}}$ & $^*73.15{\scriptstyle \pm 3.28}$ & \underline{$74.63{\scriptstyle \pm 0.54}$} & $\mathbf{86.40{\scriptstyle \pm 0.55}}$ & $48.14{\scriptstyle \pm 0.48}$ & $79.12{\scriptstyle \pm 0.73}$ \\
\rowcolor{gray!20} + \mName & $\mathbf{74.15{\scriptstyle \pm 0.58}}$ & $\mathbf{76.58{\scriptstyle \pm 1.16}}$ & \underline{$67.60{\scriptstyle \pm 0.61}$} & $75.20{\scriptstyle \pm 0.59}$ & $\mathbf{74.77{\scriptstyle \pm 0.34}}$ & \underline{$85.53{\scriptstyle \pm 0.15}$} & $\mathbf{49.67{\scriptstyle \pm 0.48}}$& $\mathbf{79.85{\scriptstyle \pm 0.54}}$ \\
\bottomrule
\end{tabular}
\end{adjustbox}
\vspace{2pt}
\caption{Generative LLM fine-tuning using LoRA. Mean $\pm$ STD of exact-match accuracy (\%) across 10 seeds. Best per column bold; second-best underlined. $^*$ indicates method includes at least 1 collapsed failure.}
\label{tab:gen}
% \vspace{-12pt}
\end{table}

%% file: tables/sequence.tex
\begin{table}[t]
\centering
\begin{adjustbox}{max width=\textwidth}
\begin{tabular}{l|cccc cccc}
\toprule
\multirow{2}{*}{Method} & \textbf{RTE} & \textbf{MRPC} & \textbf{BoolQ} & \textbf{STS-B} & \textbf{RTE} & \textbf{MRPC} & \textbf{BoolQ} & \textbf{STS-B} \\
\cmidrule(lr){2-5}\cmidrule(lr){6-9}
 & \multicolumn{4}{c}{\texttt{RoBERTa-large}} & \multicolumn{4}{c}{\texttt{Llama-3.2-1B}}\\
\midrule
FFT          & $^*76.17 {\scriptstyle \pm 15.44}$ & $^*84.90{\scriptstyle \pm 5.33}$ & $^*83.20{\scriptstyle \pm 7.39}$ & $\mathbf{91.86 {\scriptstyle \pm 0.51}}$& $\underline{79.49{\scriptstyle \pm 1.83}}$& $\mathbf{86.69{\scriptstyle \pm 1.05}}$& $^*73.98{\scriptstyle \pm 3.85}$ & $\underline{90.12{\scriptstyle \pm 0.36}}$\\
+ FocalLoss    & $77.29{\scriptstyle \pm1.70}$ & $^*86.52{\scriptstyle \pm6.50}$ & $81.22{\scriptstyle \pm 0.46}$ & $88.61{\scriptstyle \pm 2.76}$& $^*49.39{\scriptstyle \pm 2.97}$ & $^*68.38{\scriptstyle \pm 0.00}$ & $^*62.17{\scriptstyle \pm 0.00}$ & $83.66{\scriptstyle \pm 2.21}$\\
+ NoisyTune    & $^*70.65{\scriptstyle \pm 15.53}$ & ${^*86.91{\scriptstyle \pm 6.62}}$& $^*71.04{\scriptstyle \pm 11.46}$ & $91.38{\scriptstyle \pm 0.70}$& $74.69{\scriptstyle \pm 2.92}$ & $84.76{\scriptstyle \pm 2.08}$ & $67.64{\scriptstyle \pm 1.01}$ & $83.97{\scriptstyle \pm 1.81}$\\
+ SWA          & $^*78.45{\scriptstyle \pm 13.61}$ & $\mathbf{88.19{\scriptstyle \pm 1.21}}$& $^*80.36{\scriptstyle \pm 9.61}$ & $91.18{\scriptstyle \pm 1.06}$& $76.17{\scriptstyle \pm 1.47}$ & $67.29{\scriptstyle \pm 1.64}$ & \underline{$82.62{\scriptstyle \pm 3.79}$}& $85.14{\scriptstyle \pm 0.33}$\\
+ SAM          & $^*69.27{\scriptstyle \pm 10.92}$ & $^*74.83{\scriptstyle \pm 8.41}$ & $^*80.61{\scriptstyle \pm 8.34}$ & $87.58{\scriptstyle \pm 2.21}$& $75.27{\scriptstyle \pm 3.14}$ & $82.49{\scriptstyle \pm 1.84}$ & $70.07{\scriptstyle \pm 1.21}$ & $89.21{\scriptstyle \pm 0.31}$\\
+ PCGrad       & $\underline{^*81.05{\scriptstyle \pm 10.08}}$& {$^*86.67{\scriptstyle \pm 6.52}$}& \underline{$^*83.22{\scriptstyle \pm 7.40}$}& $91.04{\scriptstyle \pm 0.56}$& $73.88{\scriptstyle \pm 2.95}$& $82.84{\scriptstyle \pm 3.39}$& $67.77{\scriptstyle \pm 1.21}$& 
$90.11{\scriptstyle \pm 0.21}$\\
\rowcolor{gray!20} + \mName & $\mathbf{82.59{\scriptstyle \pm 1.24}}$ & $\underline{87.84{\scriptstyle \pm 0.74}}$& $\mathbf{85.04{\scriptstyle \pm 0.38}}$& $\underline{91.56{\scriptstyle \pm 0.36}}$& $\mathbf{82.22{\scriptstyle \pm 1.12}}$& \underline{$85.21{\scriptstyle \pm 1.39}$}& $\mathbf{84.32{\scriptstyle \pm 0.68}}$& $\mathbf{90.78 {\scriptstyle \pm 0.41}}$\\
\bottomrule
\end{tabular}
\end{adjustbox}
\vspace{2pt}
\caption{Sequence classification and regression performance across different finetuning strategies.}
\label{tab:sequence}
\vspace{-12pt}
\end{table}

%% file: tables/vision.tex
\begin{table}[t]
\centering
\begin{adjustbox}{max width=\textwidth}
\begin{tabular}{l|cccc cccc}
\toprule
\multirow{2}{*}{Method} 
& \multicolumn{4}{c}{\textbf{ CIFAR-10}}
& \multicolumn{4}{c}{\textbf{ CIFAR-100}}\\
\cmidrule(lr){2-5}\cmidrule(lr){6-9}
& LT-100 & LT-50 & ST-100 & ST-50 & LT-100 & LT-50 & ST-100 & ST-50 \\
\midrule
\texttt{ViT} (FFT) & $88.47{\scriptstyle \pm 1.77}$  & $96.46{\scriptstyle \pm 0.44}$  & \underline{$83.27{\scriptstyle \pm 2.13}$}& $96.09{\scriptstyle \pm 0.51}$ & \underline{$69.55{\scriptstyle \pm 2.71}$}& $85.53{\scriptstyle \pm 0.48}$  & $56.73{\scriptstyle \pm 1.58}$    & $84.30{\scriptstyle \pm 0.48}$ \\
+ FocalLoss & \underline{$89.32{\scriptstyle \pm 0.68}$}& $95.29{\scriptstyle \pm 3.58}$  & $^*70.31{\scriptstyle \pm 14.22}$ & \underline{$97.15{\scriptstyle \pm 0.57}$}& $^*64.68{\scriptstyle \pm 10.70}$ & $87.20{\scriptstyle \pm 2.21}$ & $57.19{\scriptstyle \pm 2.49}$  & $^*86.60{\scriptstyle \pm 4.17}$ \\
+ NoisyTune & $87.88{\scriptstyle \pm 1.67}$  & $^*88.94{\scriptstyle \pm 9.39}$ & $^*68.86{\scriptstyle \pm 15.40}$ & $96.80{\scriptstyle \pm 0.24}$ & $^*68.05{\scriptstyle \pm 3.75}$  & \underline{$87.89{\scriptstyle \pm 0.24}$}& $57.17{\scriptstyle \pm 1.11}$    & \underline{$87.89{\scriptstyle \pm 0.24}$}\\
+ SWA& $^*39.29{\scriptstyle \pm 4.29}$& $^*90.58{\scriptstyle \pm 5.31}$& $^*45.66{\scriptstyle \pm 4.80}$& $^*71.37{\scriptstyle \pm 16.83}$& $^*51.14{\scriptstyle \pm 9.60}$& $^*66.96{\scriptstyle \pm 13.46}$& $^*35.62{\scriptstyle \pm 18.60}$&$^*73.08{\scriptstyle \pm 15.65}$\\
+ SAM       & $^*88.07{\scriptstyle \pm 9.41}$ & $^*92.15{\scriptstyle \pm 11.55}$ & $82.77{\scriptstyle \pm 1.12}$& $96.57 {\scriptstyle \pm 0.22}$& $^*63.55{\scriptstyle \pm 11.08}$ & $87.81 {\scriptstyle \pm 0.91}$& \underline{$58.50 {\scriptstyle \pm 1.60}$}& $87.84{\scriptstyle \pm 0.02}$\\
+ PCGrad    & $86.63{\scriptstyle \pm 2.44}$  & \underline{$96.85{\scriptstyle \pm 0.19}$}& $^*82.05{\scriptstyle \pm 3.57}$    & $96.84{\scriptstyle \pm 0.15}$ & $68.13{\scriptstyle \pm 1.47}$  & $87.02{\scriptstyle \pm 0.37}$  & $50.01{\scriptstyle \pm 1.18}$    & $87.24{\scriptstyle \pm 0.37}$ \\
\rowcolor{gray!20} + \mName & $\mathbf{89.83{\scriptstyle \pm 0.88}}$& $\mathbf{96.97{\scriptstyle \pm 0.13}}$& $\mathbf{85.17{\scriptstyle \pm 0.90}}$& $\mathbf{97.56
{\scriptstyle \pm 0.13}}$& $\mathbf{70.59{\scriptstyle \pm 0.38}}$& $\mathbf{88.87{\scriptstyle \pm 0.09}}$& $\mathbf{60.17{\scriptstyle \pm 0.43}}$& $\mathbf{87.91{\scriptstyle \pm 0.09}}$\\
\bottomrule
\end{tabular}
\end{adjustbox}
\vspace{2pt}
\caption{Image classification on imbalanced CIFAR using \texttt{ViT-base}.
% Mean $\pm$ STD of accuracy (\%) across 10 seeds. Best per column bold;
% second-best underlined. $^*$ indicates method includes at least 1 collapsed failure.
LT = long-tailed imbalance, ST = step imbalance; the number after each denotes the imbalance ratio.}
\label{tab:image_cls}
\vspace{-12pt}
\end{table}

%% file: tables/noisy_vision.tex
\begin{table}[t]
\centering
\begin{adjustbox}{max width=\textwidth}
\begin{tabular}{l|cccccccc}
\toprule
& \multicolumn{4}{c}{Noisy Label} & \multicolumn{4}{c}{Imbalance}\\
\cmidrule(lr){2-5}\cmidrule(lr){6-9}
Ratio      & \multicolumn{1}{c}{0.2} & \multicolumn{1}{c}{0.4} & \multicolumn{1}{c}{0.6} & \multicolumn{1}{c}{0.8} &      LT-100& LT-50& ST-100&ST-50\\
\midrule
Method& \multicolumn{8}{c}{\textbf{CIFAR-10}}\\
\midrule
\texttt{Resnet-18}  & $86.47 {\scriptstyle \pm 0.28}$& $81.91 {\scriptstyle \pm 0.54}$& $73.27 {\scriptstyle \pm 0.65}$& $43.51 {\scriptstyle \pm 0.83}$ & $\mathbf{67.97{\scriptstyle \pm 0.78}}$& $93.01{\scriptstyle \pm 0.16}$& $61.15{\scriptstyle \pm 0.71}$&$93.03{\scriptstyle \pm 0.14}$\\
\rowcolor{gray!20} + \text{\mName}  & $\mathbf{87.06 {\scriptstyle \pm 0.34}}$& $\mathbf{82.59{\scriptstyle \pm 0.29}}$& $\mathbf{74.36{\scriptstyle \pm 0.79}}$& $\mathbf{44.19{\scriptstyle \pm 0.49}}$ & $67.53{\scriptstyle \pm 0.21}$& $\mathbf{93.03{\scriptstyle \pm0.08}}$& $\mathbf{61.62{\scriptstyle \pm 0.56}}$&$\mathbf{93.11{\scriptstyle \pm0.13}}$\\
\texttt{Resnet-32}  & $86.77 {\scriptstyle \pm 0.27}$& $82.51{\scriptstyle \pm 0.73}$& $74.08 {\scriptstyle \pm 0.94}$& $40.73{\scriptstyle \pm 2.21}$ & $65.80{\scriptstyle \pm 1.04}$& $\mathbf{92.93{\scriptstyle \pm 0.44}}$& $62.13{\scriptstyle \pm 0.93}$&$\mathbf{93.08{\scriptstyle \pm 0.16}}$\\
\rowcolor{gray!20} + \mName & $\mathbf{87.47{\scriptstyle \pm 0.25}}$& $\mathbf{83.20{\scriptstyle \pm 0.48}}$& $\mathbf{75.08 {\scriptstyle \pm 0.87}}$& $\mathbf{41.67{\scriptstyle \pm 2.17}}$ & $\mathbf{67.22{\scriptstyle \pm 0.99}}$& $92.71{\scriptstyle \pm 0.31}$& $\mathbf{63.69{\scriptstyle \pm 0.16}}$&$93.00{\scriptstyle \pm 0.06}$\\
\midrule
Method& \multicolumn{8}{c}{\textbf{CIFAR-100}}\\
% & \texttt{ViT-base}& $96.70{\scriptstyle \pm 0.24}$& & &\\
% & + \mName & $96.86{\scriptstyle \pm 0.05}$& & &\\
\midrule
\texttt{Resnet-18}  & $63.66 _ {\pm 0.40}$& $50.80 {\scriptstyle \pm 0.74}$& $39.41 {\scriptstyle \pm 0.57}$& $19.42{\scriptstyle \pm 0.76}$ & $36.80{\scriptstyle \pm 0.24}$ & $71.98{\scriptstyle \pm 0.41}$ & $41.06{\scriptstyle \pm 0.22}$ &$71.67{\scriptstyle \pm 0.09}$ \\
\rowcolor{gray!20} + \mName & $\mathbf{64.12 {\scriptstyle \pm 0.33}}$& $\mathbf{52.47{\scriptstyle \pm 0.72}}$& $\mathbf{40.99{\scriptstyle \pm 0.74}}$& $\mathbf{20.46{\scriptstyle \pm 0.53}}$ & $\mathbf{37.29{\scriptstyle \pm 0.78}}$& $\mathbf{72.10{\scriptstyle \pm 0.51}}$& $\mathbf{41.16{\scriptstyle \pm0.33}}$&$\mathbf{71.80{\scriptstyle \pm 0.19}}$\\
\texttt{Resnet-32}  & $65.50{\scriptstyle \pm 0.44}$& $52.62{\scriptstyle \pm 0.75}$& $40.65 {\scriptstyle \pm 1.42}$& $19.29{\scriptstyle \pm 0.68}
$ & $35.31{\scriptstyle \pm 0.95}$& $71.96{\scriptstyle \pm 0.37}$& $40.88{\scriptstyle \pm 0.09}$&$71.78{\scriptstyle \pm 0.46}$\\
\rowcolor{gray!20} + \mName & $\mathbf{65.76{\scriptstyle \pm 0.50}}$& $\mathbf{53.77 {\scriptstyle \pm 0.42}}$& $\mathbf{42.33{\scriptstyle \pm 0.66}}$& $\mathbf{20.12 {\scriptstyle \pm 0.46}}$ & $\mathbf{35.96{\scriptstyle \pm 1.07}}$& $\mathbf{72.29{\scriptstyle \pm 0.30}}$& $\mathbf{41.10{\scriptstyle \pm 0.40}}$&$\mathbf{72.07{\scriptstyle \pm 0.05}}$\\
% & \texttt{ViT-base}& & & &\\
% & + \mName & & & &\\
\bottomrule
\end{tabular}
\end{adjustbox}
\vspace{2pt}
\caption{\texttt{Resnet-\{18,32\}} performance using normal training baseline and + \mName (m=2, k=1) on CIFAR with \textbf{noisy label} and \textbf{imbalance} settings. 
% Accuracy and Std are taken over 10 runs.
}
\label{resnet:noisylabel_imbalance}
\vspace{-12pt}
\end{table}

%% file: sections/conclusion.tex
\section{Limitation}
\mName addresses seed-induced variance through gradient
aggregation; it does not directly address other instability sources
such as prompt formatting in prompt tuning~\citep{he2024does} or in-context-example selection~\citep{gupta2023coverage}.
Our stability proof requires smooth activations and an absolutely
continuous data density, plus a generic-position assumption
(Assumption~\ref{ass:transversality}) that holds for almost every
parameter setting of standard overparameterized networks but does
not extend to piecewise-linear networks (e.g., ReLU) without
additional work. The min-norm QP is cheap in our default setting
($m\!\leq\!8$), but its $\mathcal{O}(m^2)$ cost makes very large group counts impractical without a sparse-support approximation. The
empirical instability quantity we report (test-metric standard
deviation across seeds) is not the same as the uniform-stability
bound proved in Theorem~\ref{thm:bounds}; we treat the two as
complementary rather than as a formal reduction. Finally, our
sensitivity analysis covers $m\!\in\!\{2,4\}$ and
$k\!\in\!\{1,2,4\}$; behavior at larger $m$ or $k$, and on
architectures we did not test (notably Mamba and other
state-space models), is left to future work.

\section{Conclusion}
\label{sec:conclusion}

We introduced \mName, an optimizer-agnostic training procedure that replaces the arithmetic mean used by mini-batch optimization with a min-norm convex combination of group-wise gradients, applied at both intra-batch and inter-batch levels. \mName carries three guarantees that hold for any smooth model trained on data with a continuous density: zero group-wise gradient conflict, $\mathcal{O}(1/T)$ convergence, and a strictly tighter uniform-stability bound than SGD. Empirically across generative LLM fine-tuning, sequence and image classification, and regression at LPM scale, \mName is the only method in our suite that avoids collapsed failure entirely while consistently improving mean performance and reducing run-to-run variance, at a per-iteration cost indistinguishable from standard data-parallel training.

%% file: sections/appendix.tex
\input{sections/apx-related}

\input{sections/apx-algorithm}

\input{sections/apx-proof}

\input{sections/apx-exp}

%% file: sections/apx-related.tex
\section{Related Work}
\label{sec:related}

\paragraph{Learning instability across settings.}
Run-to-run variance has been documented across the ML stack. \citet{madhyastha2019model} and \citet{bethard2022we} study seed sensitivity in classical NLP pipelines. \citet{picard2021torch} and \citet{summers2021nondeterminism} show substantial variance in from-scratch training of computer-vision models attributable to seed choice and hardware non-determinism. In the fine-tuning of LPMs, \citet{devlin2019bert,dodge2020fine,lee2020mixout} attribute instability to catastrophic forgetting and small fine-tuning datasets, while \citet{mosbach2020stability} point to optimization difficulties and vanishing gradients and recommend small learning rates, bias-corrected Adam, and long training schedules; we follow these guidelines in our FFT baseline and still observe high variance. \citet{wang-etal-2023-two} formally decompose LPM fine-tuning variance into sampling and optimization components and show that the optimization component grows with overparameterization. The common thread is that stochasticity in the optimization trajectory matters more when the loss landscape admits many near-equivalent solutions which is the overparameterized regime.

\paragraph{Ensembles and weight averaging.}
Ensembles~\citep{wang2020wisdom,wang-etal-2023-two} are the most reliable mitigation but scale training cost with ensemble size. Model Soups~\citep{wortsman2022model} and Model Stock~\citep{jang2024model} average weights of multiple fine-tuned models and belong to the same cost category. Stochastic weight averaging~\citep{izmailov2018averaging,gao2022revisiting, madhyastha2019model,nishida-etal-2025-instability} avoids training multiple models but requires extra storage for checkpoints and can fail when the local basin itself is degenerate.

\paragraph{Noise injection and sharpness-aware methods.}
LNSR~\citep{hua2021noise} and NoisyTune~\citep{wu2022noisytune} inject noise into weights or representations. Sharpness-aware methods (SAM~\citep{foret2020sharpness}, ASAM~\citep{kwon2021asam}, Tilted-SAM~\citep{li2024tilted}) minimize loss under worst-case parameter perturbations and require an extra forward--backward pass per step. Regularization methods such as label smoothing~\citep{muller2019does}, Mixout~\citep{lee2020mixout}, Dropout~\citep{srivastava2014dropout}, and R-Drop~\citep{wu2021r} operate on weights or outputs rather than on gradient aggregation.

\paragraph{Gradient conflict resolvers.}
PCGrad~\citep{yu2020gradient}, CAGrad~\citep{liu2021conflict}, GradDrop \citep{chen2020just}, and MGDA~\citep{sener2018multi} address task-level gradient conflict in multitask learning. We adapt these methods to the single-task setting by grouping examples within a mini-batch. They act only on gradients with significantly negative cosine similarity, miss orthogonal or small-magnitude conflicts, and require a backward pass per group. GHM~\citep{li2019gradient} and OHEM~\citep{shrivastava2016training} reweight gradients by difficulty but do not target seed-induced instability. In contrast, \mName operates on any group partition, uses a single backward pass, and admits a clean stability guarantee.

%% file: sections/apx-algorithm.tex
\section{Algorithm}
\label{apx:algo}
Algorithm~\ref{alg:grain} summarizes intra-batch GRAIN. The inter-batch
variant is obtained by replacing the group construction in line~3 with
a ring buffer of past $k$ gradients.

\begin{algorithm}[h]
\caption{GRAIN (intra-batch variant)}
\label{alg:grain}
\begin{algorithmic}[1]
\Require initial parameters $\theta_0$; learning rate $\eta$; group count $m$; iterations $T$
\For{$t=0,1,\ldots,T-1$}
  \State sample mini-batch $\mathcal{B}_t\sim\mathcal{D}^B$
  \State partition $\mathcal{B}_t$ into $m$ disjoint groups $\{\mathcal{B}_{t,i}\}_{i=1}^m$
  \State compute group gradients $g_i=\tfrac{1}{|\mathcal{B}_{t,i}|}\sum_{z\in\mathcal{B}_{t,i}}\nabla_\theta\mathcal{L}(\theta_t,z)$
  \State solve $\alpha^\star_t\leftarrow\arg\min_{\alpha\in\Delta^{m-1}}\|\sum_i\alpha g_i\|^2$ (Frank--Wolfe)
  \State $\bar g_t\leftarrow\sum_i\alpha^\star_{t,i} g_i$
  \State $\theta_{t+1}\leftarrow\theta_t-\eta\,\bar g_t$
\EndFor
\State \Return $\theta_T$
\end{algorithmic}
\end{algorithm}

%% file: sections/apx-proof.tex
\section{Lemmas, Theorems, and Proofs}
\label{apx:theory_proofs}
 
We use $\eta_t$ for the learning rate at step $t$, $L_s$ for the
smoothness constant (Lipschitz constant of $\nabla\mathcal{L}$), and $H$ for
the Lipschitz constant of $\mathcal{L}$ in $\theta$ (so
$\|\nabla_\theta\mathcal{L}\|\!\leq\!H$).
 
\subsection{Preliminaries}
 
\begin{definition}[$L_s$-smooth function]
$f:\mathbb{R}^d\!\to\!\mathbb{R}$ is $L_s$-smooth if $\nabla f$ is
$L_s$-Lipschitz, equivalently
$|f(y)-f(x)-\langle\nabla f(x),y-x\rangle|\!\leq\!\tfrac{L_s}{2}\|y-x\|^2$.
\end{definition}
 
\begin{definition}[$H$-Lipschitz function]
$f:\Theta\times\mathcal{Z}\!\to\!\mathbb{R}$ is $H$-Lipschitz in $\theta$ if
$\|\nabla_\theta f(\theta,z)\|\!\leq\!H$ for all $\theta,z$, which
implies $|f(\theta,z)-f(\theta',z)|\!\leq\!H\|\theta-\theta'\|$.
\end{definition}
 
\begin{lemma}[Descent lemma~\citep{nesterov2013introductory}]
\label{lem:descent}
For an $L_s$-smooth $f$ and any $x,y$,
$f(y)\leq f(x)+\langle\nabla f(x),y-x\rangle+\tfrac{L_s}{2}\|y-x\|^2$.
\end{lemma}
 
\begin{lemma}[Min-norm KKT]
\label{lem:minnorm-kkt}
Let $\mathcal{G}=\{g=\sum_i \alpha_i g_i:\alpha\in\Delta^{m-1}\}$ and
$g^\star=\sum_i\alpha^\star_i g_i$ where
$\alpha^\star=\arg\min_{\alpha\in\Delta^{m-1}}\|g\|^2$. Then for all
$g\in\mathcal{G}$, $\langle g,g^\star\rangle\geq\|g^\star\|^2$. In
particular, $\langle g_i,g^\star\rangle\geq\|g^\star\|^2$ for every
group gradient $g_i$.
\end{lemma}
 
\begin{proof}
Convexity of $\mathcal{G}$ gives $g^\star+\epsilon\Delta\in\mathcal{G}$
for $\Delta=g-g^\star$ and $\epsilon\in[0,1]$. Optimality of
$g^\star$ implies $\|g^\star+\epsilon\Delta\|^2\geq\|g^\star\|^2$,
i.e.\ $2\epsilon\langle\Delta,g^\star\rangle+\epsilon^2\|\Delta\|^2\geq 0$.
Letting $\epsilon\!\to\!0^+$ yields
$\langle\Delta,g^\star\rangle\geq 0$, hence
$\langle g,g^\star\rangle\geq\|g^\star\|^2$. The statement for $g_i$
follows from $g_i\in\mathcal{G}$ (set $\alpha_j=\delta_{ij}$).
\end{proof}
 
\subsection{Proof of Theorem~\ref{thm:zero_conflict}}
 
Lemma~\ref{lem:minnorm-kkt} applied to $g^\star=\bar g$ gives
$\langle g_i,\bar g\rangle\geq\|\bar g\|^2\geq 0$ for every $i$.\qed
 
\subsection{Proof of Theorem~\ref{thm:convergence}}
 
\begin{proof}
Apply Lemma~\ref{lem:descent} at $\theta_{t+1}=\theta_t-\eta\bar g_t$:
\[
\mathcal{L}(\theta_{t+1})\leq\mathcal{L}(\theta_t)-\eta\langle\nabla\mathcal{L}(\theta_t),\bar g_t\rangle
+\tfrac{L_s\eta^2}{2}\|\bar g_t\|^2.
\]
Writing $\nabla\mathcal{L}(\theta_t)=\tfrac{1}{m}\sum_i g_i$ and using
Theorem~\ref{thm:zero_conflict},
$\langle\nabla\mathcal{L}(\theta_t),\bar g_t\rangle
=\tfrac{1}{m}\sum_i\langle g_i,\bar g_t\rangle\geq\|\bar g_t\|^2$.
Hence
\[
\mathcal{L}(\theta_{t+1})\leq\mathcal{L}(\theta_t)-\big(\eta-\tfrac{L_s\eta^2}{2}\big)\|\bar g_t\|^2.
\]
For $0<\eta<1/L_s$ the coefficient is positive. Telescoping from
$t=0$ to $T-1$:
\[
\sum_{t=0}^{T-1}\|\bar g_t\|^2
\leq\frac{2(\mathcal{L}(\theta_0)-\mathcal{L}^\star)}{\eta-L_s\eta^2/2},
\quad
\min_{0\leq t<T}\|\bar g_t\|^2\leq\frac{2(\mathcal{L}(\theta_0)-\mathcal{L}^\star)}{(\eta-L_s\eta^2/2)\,T}.
\]
For i.i.d.\ groups, $\mathbb{E}[\bar g_t]=\mathbb{E}[\nabla\mathcal{L}(\theta_t)]$
(since each $\lambda^\star_t$ is invariant under joint relabeling of
groups), so $\bar g_t\!\to\!0$ implies $\nabla\mathcal{L}(\theta_t)\!\to\!0$.
\end{proof}
 
\subsection{Tools for measure-theoretic statements}

\textbf{Notation:} $\lambda^m(f)$ denotes that $f$ has $m$-dimensional Lebesgue measure zero.
 
\begin{lemma}[$C^1$ image of compact cube has measure zero]
\label{lem:c1-image}
Let $K\subset\mathbb{R}^k$ be a compact cube and $f:K\!\to\!\mathbb{R}^m$
a $C^1$ map with $k<m$. Then the Lebesgue $\lambda^m(f(K))=0$.
\end{lemma}
 
\begin{proof}
$Df$ is bounded on $K$ by some $L_f$, so $f$ is Lipschitz with
constant $L_f$. Fix $N\!\in\!\mathbb{N}$ and cover $K$ by $N^k$
sub-cubes of side $\delta=\operatorname{diam}(K)/N$. Each sub-cube
$Q_i$ has diameter $\sqrt{k}\,\delta$, so $f(Q_i)$ is contained in a
ball of radius $L_f\sqrt{k}\,\delta$, which fits inside an
$m$-dimensional cube of side $2L_f\sqrt{k}\,\delta$. Summing,
$\lambda^m(f(K))\leq N^k(2L_f\sqrt{k}\,\delta)^m
= C\cdot N^{k-m}$ for $C=(2L_f\sqrt{k}\,\operatorname{diam}(K))^m$.
Since $k<m$, $N^{k-m}\!\to\!0$ as $N\!\to\!\infty$, so the Lebesgue
$\lambda^m(f(K))=0$.
\end{proof}
 
\begin{theorem}[Submanifolds of positive codimension have measure zero]
\label{thm:submanifold-measure}
Let $M\subset\mathbb{R}^m$ be a $k$-dimensional $C^1$ submanifold with
$k<m$. Then the Lebesgue $\lambda^m(M)=0$.
\end{theorem}
 
\begin{proof}
$M$ admits a countable atlas
$\{f_i:K_i\!\to\!\mathbb{R}^m\}_{i\geq 1}$ with $K_i\subset\mathbb{R}^k$
compact (using second-countability of $M$ and compact exhaustion of
chart domains). By Lemma~\ref{lem:c1-image},
$\lambda^m(f_i(K_i))=0$. Countable sub-additivity gives the Lebesgue
$\lambda^m(M)\leq\sum_i\lambda^m(f_i(K_i))=0$.
\end{proof}
 
\subsection{Min-norm differs from the mean almost surely}
\label{apx:proof-minnorm}
 
\begin{theorem}[Min-norm differs from the mean, a.s.]
\label{thm:minnorm-neq-mean}
Let $f_\theta$ be a model with smooth activations and let
$z_1,\ldots,z_m$ be i.i.d.\ samples from a distribution with
continuous positive density $p$ on $\mathcal{X}$. Define
$g_i=\nabla_\theta\mathcal{L}(\theta,z_i)$ and
$\alpha^\star=\arg\min_{\alpha\in\Delta^{m-1}}\|\sum_i\alpha_i g_i\|^2$.
Under Assumption~\ref{ass:transversality} below
(satisfied for almost every $\theta$ in the model parameter space),
for almost every $\theta$,
\[
\Pr\!\left(\alpha^\star=\tfrac{1}{m}\mathbf{1}\right) = 0,
\]
where the probability is taken over the data sampling.
\end{theorem}
 
\begin{corollary}[Strict gradient-norm separation]
\label{cor:strict-norm}
Under the conditions of Theorem~\ref{thm:minnorm-neq-mean},
$\|\bar g\|<\|\tfrac{1}{m}\sum_{i=1}^m g_i\|$ almost surely.
\end{corollary}
 
\begin{proof}[Proof of Corollary~\ref{cor:strict-norm}]
Min-norm optimality of $\alpha^\star$ gives
$\|\bar g\|\leq\|\tfrac{1}{m}\sum_i g_i\|$ deterministically. The
inequality is strict whenever $\alpha^\star\neq\tfrac{1}{m}\mathbf{1}$,
which by Theorem~\ref{thm:minnorm-neq-mean} happens almost surely.
\end{proof}
 
\paragraph{Proof of Theorem~\ref{thm:minnorm-neq-mean}.}
The proof requires care: a direct application of Sard's theorem to
$\Psi$ (the constraint map) shows that the set of \emph{critical
values} of $\Psi$ in $\mathbb{R}^{m-1}$ has measure zero, which does
\emph{not} imply that the specific value $0\in\mathbb{R}^{m-1}$ is a
regular value. The correct tool is the parametric transversality
theorem.
 
\paragraph{Setup.}
Fix a model architecture and let $g_i(\theta,z)=\nabla_\theta\mathcal{L}(\theta,z_i)\in\mathbb{R}^d$
be the per-example gradient at parameter $\theta\in\Theta$ and data
point $z_i\in\mathcal{X}$. Smooth activations make $g_i$ a $C^\infty$
function of $(\theta,z_i)$. Define
$\bar g_{\rm mean}=\tfrac{1}{m}\sum_i g_i$ and
$\Phi_i(\theta;z_1,\ldots,z_m)=\langle g_i,\bar g_{\rm mean}\rangle-\|\bar g_{\rm mean}\|^2$
for $i=1,\ldots,m$. Since $\sum_i\Phi_i\equiv 0$ identically, only
$m-1$ are independent; collect them into
$\Psi_\theta:\mathcal{X}^m\!\to\!\mathbb{R}^{m-1}$,
$\Psi_\theta=(\Phi_1,\ldots,\Phi_{m-1})$. By KKT, the equal-weight
event is exactly
$\{\lambda^\star=\tfrac{1}{m}\mathbf{1}\}=\Psi_\theta^{-1}(0)$.
$\Psi_\theta\in C^\infty(\mathcal{X}^m;\mathbb{R}^{m-1})$ for each fixed $\theta$.
 
\paragraph{Joint map.}
Define
\[
F:\Theta\times\mathcal{X}^m\!\to\!\mathbb{R}^{m-1},
\qquad
F(\theta,z_1,\ldots,z_m)=\Psi_\theta(z_1,\ldots,z_m).
\]
$F$ is $C^\infty$ jointly in $(\theta,z)$ by the smoothness assumption.
 
\begin{assumption}[Generic position of the parametric family]
\label{ass:transversality}
$F$ is a submersion at every point of $F^{-1}(0)$, i.e.\ the
differential $dF$ has full rank $m-1$ everywhere on $F^{-1}(0)$.
Equivalently, $0$ is a regular value of the joint map $F$.
\end{assumption}
 
This is a mild genericity condition: when the parametric family is
overparameterized ($\dim\Theta\!\gg\!d$) and has continuously varying
gradients, $dF$ generically has full row rank. It is not a strong
restriction in practice; in particular, for almost every randomly
initialized smooth network, Assumption~\ref{ass:transversality}
holds. We treat it as standing throughout this section.
 
\paragraph{Parametric transversality.}
We invoke the following standard result.
 
\begin{theorem}[Parametric transversality;
{\citealt[Thm.~3.2.7]{hirsch2012differential}}]
\label{thm:parametric-transversality}
Let $\Theta$, $\mathcal{X}^m$, and $\mathbb{R}^{m-1}$ be smooth manifolds and
$F:\Theta\times\mathcal{X}^m\!\to\!\mathbb{R}^{m-1}$ a $C^\infty$ map for which
$y\in\mathbb{R}^{m-1}$ is a regular value. Then for almost every
$\theta\in\Theta$ (in any measure absolutely continuous w.r.t.\
Lebesgue measure on $\Theta$), $y$ is a regular value of the
restricted map $\Psi_\theta:=F(\theta,\cdot):\mathcal{X}^m\!\to\!\mathbb{R}^{m-1}$.
\end{theorem}
 
Combining Assumption~\ref{ass:transversality} (which makes $0$ a
regular value of $F$) with
Theorem~\ref{thm:parametric-transversality} yields: for almost
every $\theta\in\Theta$, $0$ is a regular value of $\Psi_\theta$.
 
\paragraph{Conclusion.}
Fix any such $\theta$ in the full-measure subset of $\Theta$
furnished by parametric transversality. By the preimage theorem,
$\Psi_\theta^{-1}(0)$ is a smooth submanifold of $\mathcal{X}^m$ of dimension
$mD-(m-1)$, where $D=\dim\mathcal{X}$. Since $mD-(m-1)<mD$,
Theorem~\ref{thm:submanifold-measure} gives
$\lambda^{mD}(\Psi_\theta^{-1}(0))=0$.
 
The data law $z_1,\ldots,z_m\stackrel{\text{i.i.d.}}{\sim}p$ has joint
density $\prod_i p(z_i)$, which is absolutely continuous w.r.t.\
Lebesgue measure on $\mathcal{X}^m$. Absolutely continuous measures assign
zero probability to Lebesgue-null sets, so
\[
\Pr\!\left(\lambda^\star=\tfrac{1}{m}\mathbf{1}\right)
=\Pr\!\left((z_1,\ldots,z_m)\in\Psi_\theta^{-1}(0)\right)
=\int_{\Psi_\theta^{-1}(0)}\prod_i p(z_i)\,dz_1\cdots dz_m
=0.
\]
This proves Theorem~\ref{thm:minnorm-neq-mean}; the gradient-norm
strict inequality $\|\bar g\|<\|\tfrac{1}{m}\sum_i g_i\|$ a.s.\ follows
as Corollary~\ref{cor:strict-norm}.\qed
 
\subsection{Strictly tighter expansion bound}
 
\begin{lemma}[Strict step-size separation]
\label{lem:step-separation}
Under the assumptions of Theorem~\ref{thm:minnorm-neq-mean}, with
probability one over the data sampling,
\[
\|\theta-U_{\mathrm{GRAIN}}(\theta)\|
\;<\;
\|\theta-U_{\mathrm{SGD}}(\theta)\|
\;\leq\;\eta H,
\]
where $U_{\mathrm{GRAIN}}(\theta)=\theta-\eta\bar g$ and
$U_{\mathrm{SGD}}(\theta)=\theta-\eta\,\tfrac{1}{m}\sum_i g_i$.
\end{lemma}
 
\begin{proof}
$\|\theta-U_{\mathrm{GRAIN}}(\theta)\|=\eta\|\bar g\|$,
$\|\theta-U_{\mathrm{SGD}}(\theta)\|=\eta\|\tfrac{1}{m}\sum_i g_i\|$.
The strict inequality $\|\bar g\|<\|\tfrac{1}{m}\sum_i g_i\|$ a.s.\
is Corollary~\ref{cor:strict-norm}. The upper bound $\eta H$ follows
from $\|g_i\|\leq H$ for every $i$ and convexity of the norm.
\end{proof}
 
\subsection{Proof of Theorem~\ref{thm:bounds}}
 
\begin{proof}
Fix a data sampling sequence $\mathcal{B}_0,\ldots,\mathcal{B}_{T-1}$ shared
between the two runs (this is the standard data-coupled
Hardt--Recht--Singer setup). Let
$\Delta_t=\|\theta_t-\theta'_t\|$ for two GRAIN runs from
initializations $\theta_0\neq\theta'_0$, and
$\Delta'_t=\|\psi_t-\psi'_t\|$ for the corresponding SGD runs.
For SGD,
\[
\Delta'_{t+1}
=\|\psi_t-\psi'_t-\eta_t(g_{t,\mathrm{SGD}}-g'_{t,\mathrm{SGD}})\|
\leq\Delta'_t+\eta_t(\|g_{t,\mathrm{SGD}}\|+\|g'_{t,\mathrm{SGD}}\|)
\leq\Delta'_t+2\eta_t H.
\]
Iterating: $\Delta'_T\leq\|\theta_0-\theta'_0\|+2H\sum_{t=0}^{T-1}\eta_t$.
For GRAIN, the same triangle-inequality decomposition gives
\[
\Delta_{t+1}\leq\Delta_t+\eta_t(\|\bar g_t\|+\|\bar g'_t\|).
\]
By Lemma~\ref{lem:step-separation},
$\|\bar g_t\|<\|\tfrac{1}{m}\sum_i g_{t,i}\|\leq H$ a.s.\ at every
step, and likewise for the primed run. Hence
$\Delta_{t+1}<\Delta_t+2\eta_t H$ a.s.\ at every step, and
\[
\Delta_T < \|\theta_0-\theta'_0\|+2H\sum_{t=0}^{T-1}\eta_t \quad\text{a.s.}
\]
Taking expectations preserves the strict inequality because the
event of equality has probability zero. By $H$-Lipschitzness of
$\mathcal{L}$, $\operatorname{Ins}_{z,\mathcal{A}}\leq H\,\mathbb{E}[\Delta_T]$. Combining,
\eqref{eq:bound} follows.
\end{proof}

\subsection{Proof for Theorem~\ref{thrm:sharpness_bound}}
\begin{proof}
    We have $S_\rho^{GRAIN} = \text{max}_{\|\epsilon\| \leq \rho} {\mathcal{L}}(\theta+\epsilon) - {\mathcal{L}}(\theta)$. 
    % Second order Taylor expansion gives:
    % \begin{equation}
    %     \mathcal{L}(\theta + \epsilon) \approx \mathcal{L}(\theta) + \bar{g}^\top \epsilon + \frac{1}{2}\epsilon^\top H\epsilon,
    % \end{equation}
    % where $\bar{g} = \sum_{i=1}^m \alpha^\star \nabla l_i(\theta)$ and $H$ is the Hessian.
    Since $\{l_i\}_{i=1}^m$ is $L$-smooth hence ${\mathcal{L}}(\theta)$ is $L$-smooth:
    \begin{equation}
    \begin{split}
        & l_i(\theta + \epsilon) \leq l_i(\theta) + {g_i}^\top \epsilon + \frac{L}{2} \|\epsilon\|^2, \\
        \Rightarrow &\sum_{i=1}^m \alpha^\star l_i(\theta + \epsilon) \leq \sum_{i=1}^m \alpha^\star l_i(\theta) + (\sum_{i=1}^m \alpha^\star {g_i}) ^\top \epsilon + \frac{L}{2} \|\epsilon\|^2, \\
        \Rightarrow &{\mathcal{L}}(\theta + \epsilon) \leq {\mathcal{L}}(\theta) + \bar{g}^\top \epsilon + \frac{L}{2} \|\epsilon\|^2,
    \end{split}
    \end{equation}
    hence:
    \begin{equation}
    \label{eq:sharpness_expansion}
    \begin{split}
        \text{max}_{\| \epsilon \| \leq \rho} {\mathcal{L}}(\theta + \epsilon) - {\mathcal{L}}(\theta) 
        % & \leq \text{max}_{\| \epsilon \| \leq \rho} [\bar{g}^\top \epsilon + \frac{1}{2}\epsilon^\top H\epsilon], \\
        & \leq \text{max}_{\| \epsilon \| \leq \rho} [\bar{g}^\top \epsilon + \frac{L}{2} \| \epsilon \|^2].
    \end{split}
    \end{equation}

    Cauchy-Schwarz inequality gives:
    \begin{equation}
        \bar{g}^\top\epsilon \leq \|\bar{g} \| \|\epsilon\| \leq \rho \|\bar{g} \|,
    \end{equation}
    we also have:
    \begin{equation}
        \frac{L}{2} \|\epsilon\|^2 \leq \frac{L}{2} \rho^2,
    \end{equation}
    thus:
    \begin{equation}
        S_\rho^{GRAIN}(\theta) = \text{max}_{\| \epsilon \| \leq \rho} \mathcal{L}(\theta + \epsilon) - \mathcal{L}(\theta) \leq \rho \| \bar{g} \| + \frac{\rho^2L}{2}.
    \end{equation}
    
    % Under the assumption (2) the first order term dominates hence $\epsilon^\star = \rho \frac{\bar{g}}{\| \bar{g} \|}$ maximize $\bar{g}^\top \epsilon + \frac{1}{2}\epsilon^\top H\epsilon$. Substituting into Eq~\ref{eq:sharpness_expansion}:
    % \begin{equation}
    %     \mathcal{L}(\theta + \epsilon) - \mathcal{L}(\theta) \leq \rho \|\bar{g}\|+\frac{\rho^2}{2}\frac{\bar{g}^\top H \bar{g}}{\|g\|^2_2},
    % \end{equation}
    % Each group loss is $L$-smoothness and $\sum_{i=1}^m \alpha_i^\star = 1$ we have:
    % \begin{equation}
    %     \frac{\bar{g}^\top H \bar{g}}{\|g\|^2} \leq L
    % \end{equation}
    % this gives:
    % \begin{equation}
    %     S_\rho^{GRAIN}(\theta) = \text{max}_{\| \epsilon \| \leq \rho} \mathcal{L}(\theta + \epsilon) - \mathcal{L}(\theta) \leq \rho \| \bar{g} \| + \frac{\rho^2L}{2}.
    % \end{equation}
\end{proof}

Using the same proof framework as above we also obtain: 
\begin{equation}
    S_\rho^{SGD}(\theta) \leq \rho \| g \| + \frac{\rho^2L}{2},
\end{equation}
where $g = \frac{1}{m} \sum_{i=1}^m \nabla l_i(\theta)$.

%% file: sections/apx-exp.tex
\section{Detailed Experiment Setup}
\label{apx:exp}
In this section we detail the settings for our experiments.

\subsection{Settings \& Implementations}
\label{apx:setting}
We arbitrarily choose 10 random seeds (42, 52, 62, 72, 82, 92, 102, 112, 122, 132) to obtain the mean and variance performance. All single runs are conducted on 8 NVIDIA-A100-80GB GPUs.

\textbf{Baselines:} 
\begin{itemize}
    \item For \textbf{PCGrad}~\cite{yu2020gradient} we customize it for our task where we consider each group of training example in every iteration as an individual task in multi-task learning context and apply PCGrad learning procedure. 
    % The pseudo code of PCGrad is shown in Algorithm~\ref{alg:pcgrad}.
    % \item For the \textbf{LNSR}~\cite{hua2021noise} approach, we apply a noise regularizer to the last two layers of the classifier, with $\lambda$ (the weight of the regularization loss) set to 0.2 and 0.1, respectively.
    \item For \textbf{NoisyTune}~\cite{wu2022noisytune}, we add uniform noise $U(a, b)$, where $\lambda$ controls the noise intensity. For the GLUE dataset, we set $\lambda = 0.15$, following~\cite{wu2022noisytune}. For other task we use the same $\lambda$ setting.
    \item \textbf{SWA}~\cite{izmailov2018averaging} maintains an average of model parameters over previous checkpoints. We tune the averaging window size in the range ${3, 5}$ and report the best result for each task.
    % \item The settings for the PEFT methods LoRA~\cite{hu2022lora}, AdaLoRA~\cite{zhang2023adalora}, and IA3~\cite{liu2022few} are listed in Table~\ref{table:peft_settings}. All settings follow those used in the respective reference papers.
    \item Although the original implementation of \textbf{Focal Loss}~\cite{lin2017focal} suggests default hyperparameters of $\alpha = 0.25$ and $\gamma = 2$, the modified loss for a single data point $z_i$ is defined as: $\mathcal{L}(z_i) = -\alpha_t(1-p_t)^\gamma \text{log}(p_t)$ where $p_t$ denotes the predicted probability of the ground-truth class. These configurations frequently induced training divergence or collapsed states in our preliminary evaluations. Consequently, we performed an extensive hyperparameter search for $\alpha$ and $\gamma$ tailored to each task and pretrained architecture. The results reported herein represent the optimal performance achieved after this exhaustive tuning process. 
    \item \textbf{SAM}~\cite{foret2020sharpness} add an adversarial gradient to minimize the generalization gap between the training loss and the testing one, in our experiment we set $\rho = 0.05$ as mentioned in the report.
\end{itemize}
All approaches are optimized using AdamW~\cite{loshchilov2017decoupled} optimizer.

\subsection{Generative Tasks}
\label{apx:generative_tasks}

\paragraph{Models.}
We use four open-weight LLMs spanning two model families and two parameter scales: \texttt{Qwen2-7B-base}~\citep{yang2024qwen2} and \texttt{Qwen2.5-14B-base}~\citep{yang2024qwen25} from the Qwen family, and \texttt{Mistral-7B-v0.3}~\citep{jiang2023mistral7b} and \texttt{Ministral-3-14B-Base-2512}~\citep{liu2026ministral3} from the Mistral family. All four are pre-trained base checkpoints; we deliberately avoid instruction-tuned variants to keep the loss landscape comparable across models and to isolate seed-induced training instability from instruction-tuning artifacts. 

\paragraph{Datasets.}
We evaluate on two generative reasoning benchmarks. \textbf{GSM8K}~\citep{cobbe2021training} contains 8.5K grade-school math word problems with chain-of-thought solutions; we use the standard 7.5K-train / 1.3K-test split. \textbf{PubMedQA}~\citep{jin2019pubmedqa} is a biomedical QA dataset where the model answers \emph{yes/no/maybe} given a research abstract and question; we use the expert-labeled partition (PQA-L, 1K examples) and split it 80\%-20\% into training and test sets. Both datasets exercise long-form generation conditioned on a structured input; together they cover quantitative reasoning (GSM8K) and qualitative biomedical inference (PubMedQA).

\paragraph{Fine-tuning protocol.}
We fine-tune all models with LoRA~\citep{hu2022lora} on a single $8 \times$ A100 (80GB) node. The hyperparameter settings are shown in Table~\ref{tab:gen_settings}.

\paragraph{Metrics.}
We evaluate in the zero-shot setting using \textbf{exact-match accuracy} as the primary metric. For PubMedQA, a prediction is correct if the generated answer normalizes to one of \texttt{\{yes, no, maybe\}} and matches the gold label; for GSM8K, the final numeric answer (extracted from the model's output via the standard regex) must match the gold answer.

% \paragraph{Prompts.} We format the training examples for GSM8K and PubmedQA for \texttt{Mistral} family as:
% \begin{itemize}
%     \item \textbf{GSM8K:}
%     \item \textbf{PubmedQA:}
% \end{itemize}
% For \texttt{Qwen} family, the prompt is formatted as:
% \begin{itemize}
%     \item \textbf{GSM8K:}
%     \item \textbf{PubmedQA:}
% \end{itemize}

% \paragraph{Seeds and statistical reporting.}
% For each (model, dataset, method) cell in Table~\ref{tab:gen}, we run \textbf{10 random seeds} with all other settings held fixed, and report mean and standard deviation of test-set accuracy. Following the convention introduced in Section~\ref{sec:phenomenon}, we mark a cell with $^*$ if any of its 10 runs satisfies our \emph{collapsed-failure} criterion (defined in Section~\ref{sec:phenomenon}). All baseline implementations use the publicly available author code where available, with hyperparameters tuned by grid search on a held-out validation slice; details in Section~\ref{app:baseline-tuning}.

\subsection{Sequence Classification}
\label{apx:sequence_classification}
We conduct our experiments on 4 tasks on Super-GLUE benchmark including: RTE, COPA, and BoolQ data statistics is shown in Table~\ref{tab:data_stats}.
We finetune pretrained language model with different type of architecture including encoder only: \texttt{RoBERTa-large}, decoder only: \texttt{Llama-3.2-1B}.
To ensure a proper baseline implementation, we refer to the configurations in~\cite{liu2019roberta, hu2022lora} and replicate the state-of-the-art (SOTA) scores reported using \texttt{\small RoBERTa-large} with FFT. 
For \texttt{Llama-3.2-1B}, no established reference performance was available in the literature; we therefore fine-tune it using our own settings. 
Although our implementation may not achieve SOTA performance, it establishes a consistent basis for comparison across different methods.

\subsection{Image classification}
\label{apx:image_classification}
\paragraph{Image classification under distribution shift.}
We use image classification to test GRAIN under \emph{train-test
distribution shift}, a setting known to amplify run-to-run variance
in computer vision and one that places GRAIN's variance-reduction
claim under stress. We construct shift in two complementary ways:
(i)~\textit{class-imbalanced training} (varying class frequencies)
and (ii)~\textit{noisy-label training} (corrupted supervision). In
both cases the test set is held fixed at the original clean,
class-balanced distribution, isolating the effect of the training-time
shift on stability.

\paragraph{Datasets.} CIFAR-10 and CIFAR 100~\cite{krizhevsky2009learning} contain $50{,}000$ training images and $10{,}000$ test images of size $32{\times}32$, with 10 and 100 classes respectively. The test set is left untouched in both shift scenarios.

\paragraph{Class-imbalanced training.}
Following~\citet{cao2019learning}, we reduce the number of training
samples per class while keeping the test set unchanged. We sweep two
imbalance patterns---long-tailed (exponential
decay)~\cite{cui2019class} and step
imbalance~\cite{buda2018systematic,azizzadenesheli2019regularized}---at two ratios (100:1 and 50:1), yielding four imbalanced training sets per benchmark. 
We initially screened
\texttt{ResNet-34}, \texttt{ResNet-50}, \texttt{EfficientNet-b0},
and \texttt{ViT-base} on the balanced versions and observed
negligible run-to-run variance for all models. Under imbalance,
CNN-based models showed lower mean accuracy but only modest variance
relative to \texttt{ViT-base}, consistent with prior work on
long-tailed recognition. We therefore use \texttt{ViT-base} for the
imbalance experiments, where the seed-induced variance is largest
and the stability story is most informative.

\paragraph{Noisy-label training.}
We corrupt each training label with probability
$\rho \in \{0.2, 0.4, 0.6, 0.8\}$ by replacing it with a
uniformly-chosen class other than the true one (symmetric label
noise). We train \texttt{ResNet-18} and \texttt{ResNet-32} on
CIFAR-10 and CIFAR-100 under each $\rho$, with the test set left
clean. Label noise produces a controlled form of intra-batch
gradient conflict---the gradients of correctly- and
incorrectly-labeled examples within a mini-batch point in
systematically different directions---whose intensity is set by
$\rho$, complementing the imbalance experiments (where conflict is
driven by class frequency rather than supervision noise).

We follow~\cite{dosovitskiy2021an} recommendation on finetuning \texttt{ViT} on CIFAR-10 and CIFAR-100 datasets using SGD with a learning rate of $1E-2$ and a total batch size of $128$ detailed setup can be found in Table~\ref{tab:hp_setting}.

\subsection{Additional Results and Analysis}
\label{apdix:add_results}

Figure~\ref{spec_seed_performance} presents the performance of individual random seeds for the examined models on both PubMedQA and GSM8K under SGD and \mName fine-tuning. Compared to \mName, SGD exhibits substantially higher variance, with performance fluctuating significantly across different seeds. This instability is particularly pronounced for \texttt{Mistral-7B} and \texttt{Mistral-14B} on PubMedQA, where performance ranges from approximately 10\% in failed runs to nearly 70\% in successful runs.

In contrast, \mName greatly reduces this seed sensitivity. For example, the performance of \texttt{Mistral-7B} under \mName is consistently concentrated around 70\%, matching the best performance achieved by SGD while avoiding catastrophic failures. Similar trends are observed across the other models and datasets, indicating that \mName not only improves average performance but also substantially enhances training stability and robustness to random initialization.

\begin{figure*}[h]
    \centering
    \begin{subfigure}{0.24\linewidth}
        \centering
        \includegraphics[width=\linewidth]{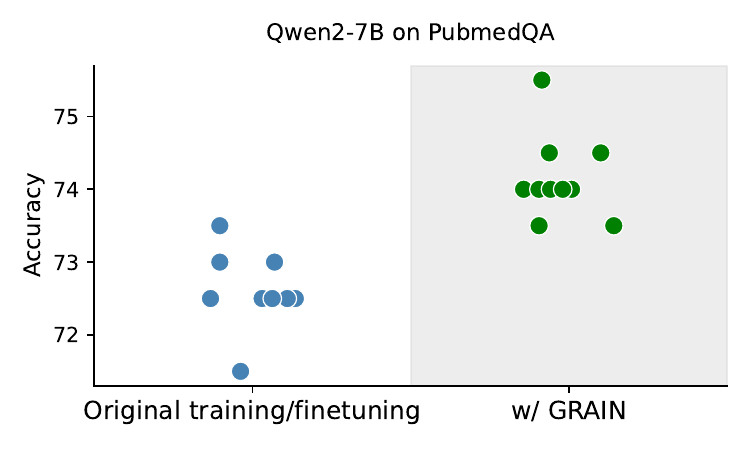}
    \end{subfigure}
    \hfill
    \begin{subfigure}{0.24\linewidth}
        \centering
        \includegraphics[width=\linewidth]{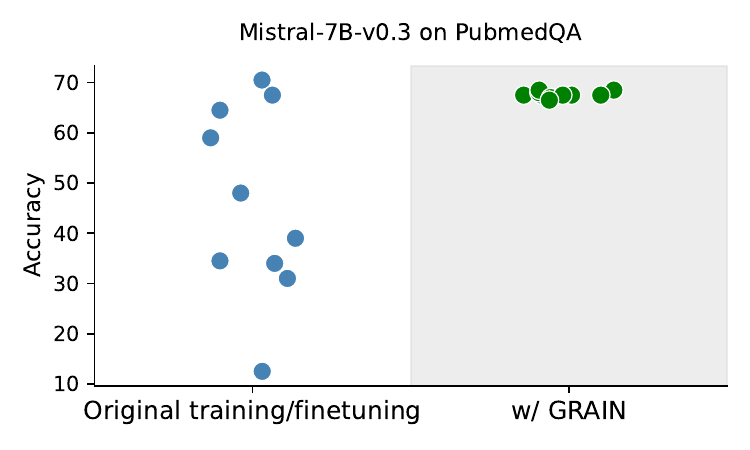}
    \end{subfigure}
    \hfill
    \begin{subfigure}{0.24\linewidth}
        \centering
        \includegraphics[width=\linewidth]{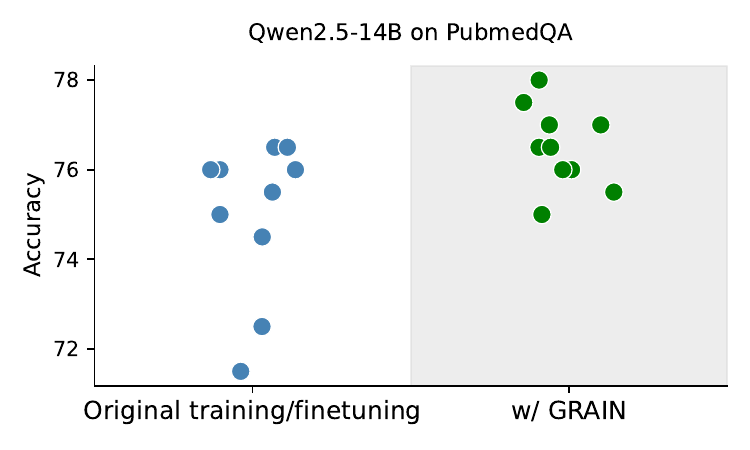}
    \end{subfigure}
    \hfill
    \begin{subfigure}{0.24\linewidth}
        \centering
        \includegraphics[width=\linewidth]{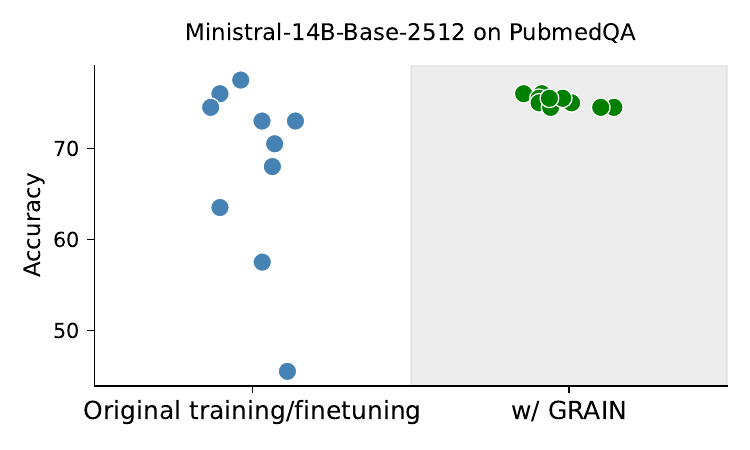}
    \end{subfigure}
    \begin{subfigure}{0.24\linewidth}
        \centering
        \includegraphics[width=\linewidth]{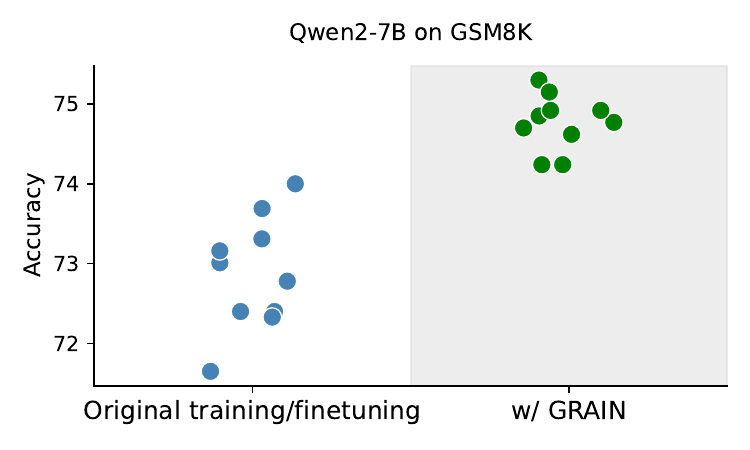}
    \end{subfigure}
    \hfill
    \begin{subfigure}{0.24\linewidth}
        \centering
        \includegraphics[width=\linewidth]{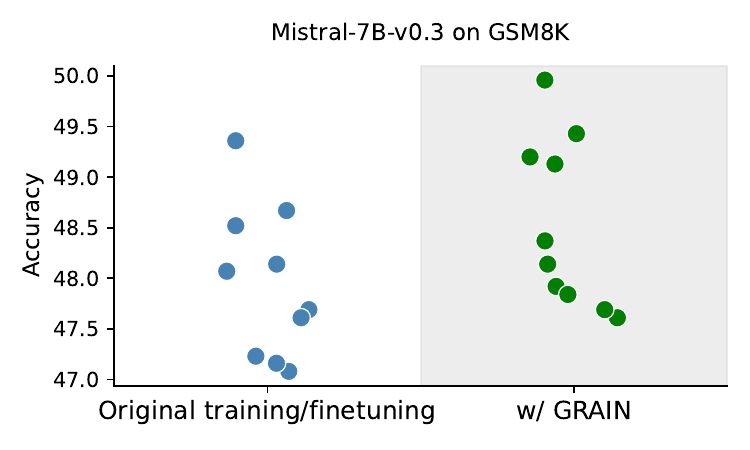}
    \end{subfigure}
    \hfill
    \begin{subfigure}{0.24\linewidth}
        \centering
        \includegraphics[width=\linewidth]{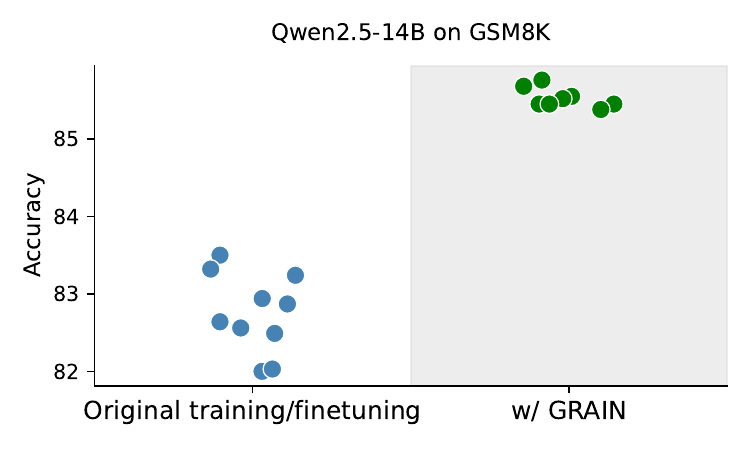}
    \end{subfigure}
    \hfill
    \begin{subfigure}{0.24\linewidth}
        \centering
        \includegraphics[width=\linewidth]{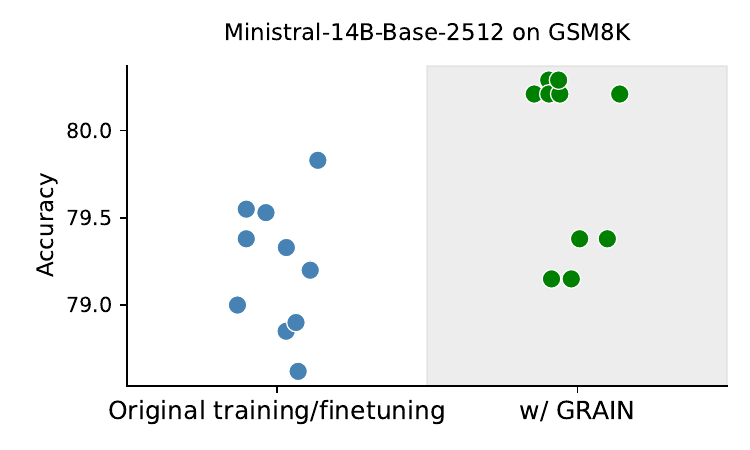}
    \end{subfigure}
    \caption{Individual seed performance of \texttt{Qwen2-7B}, \texttt{Mistral3-7B}, \texttt{Qwen2.5-14B} and \texttt{Ministral-14B-Base} (left to right) using LoRA finetuning with (green) and with out (blue) using GRAIN on PubmedQA (the first row) and GSM8K (the second row).}
    \label{spec_seed_performance}
\end{figure*}

Figures~\ref{local_level_grad} and~\ref{global_level_grad} visualize the gradient dynamics of successful and failed runs at both the local level (within a mini-batch) and the global level (across mini-batches). We observe that the cosine similarity between gradients is consistently higher in successful runs than in failed runs, indicating better gradient alignment throughout training.

At the local level, failed runs exhibit frequent gradient conflicts, leading to severe gradient cancellation. Similar behavior is observed at the global level. Although inter-batch gradient cancellation is less severe than intra-batch cancellation, as evidenced by the substantially lower frequency of negative cosine similarities between $g_p$ and $g_c$, failed runs still show considerably more gradient conflict than successful runs. In particular, from iterations 20 to 160, failed runs consistently experience inter-batch conflicts, whereas successful runs encounter fewer than five such conflicts, primarily during the early stages of training (iterations 0 to 80), after which the gradients become largely aligned.

This persistent gradient cancellation in failed runs results in parameter updates with very small magnitudes. Specifically, the norm of the parameter update remains around 5 from iterations 20 to 160. In contrast, successful runs exhibit substantially larger update norms, especially after iteration 80, enabling more effective optimization. Overall, these observations suggest that persistent gradient cancellation at both intra- and inter-batch levels can significantly hinder learning by producing ineffective parameter updates.

\begin{figure*}[t]
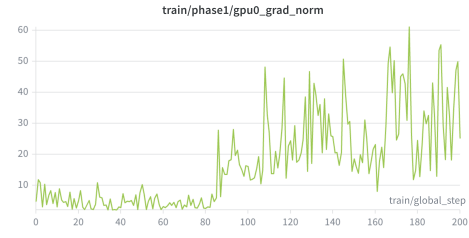
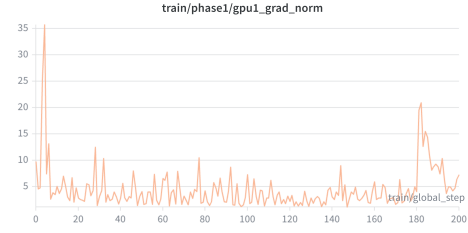
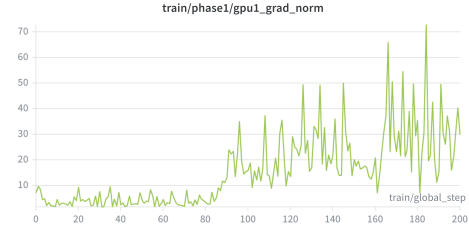
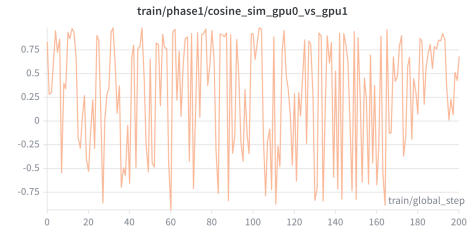
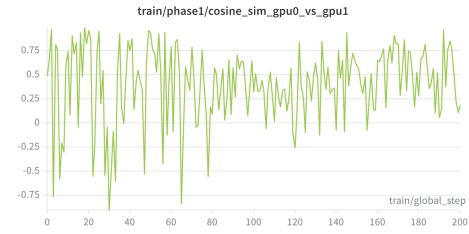

    \centering
    \begin{subfigure}{0.45\linewidth}
        \centering
        \includegraphics[width=\linewidth]{figures/fai_g0.pdf}
        \caption{Gradient norm of group $\#1: g_1$ (Local level - Failed).}
    \end{subfigure}
    \hfill
    \begin{subfigure}{0.45\linewidth}
        \centering
        \includegraphics[width=\linewidth]{figures/suc_g0.pdf}
        \caption{Gradient norm of group $\#1: g_1$ (Local level - Successful).}
    \end{subfigure}
    \hfill
    \begin{subfigure}{0.45\linewidth}
        \centering
        \includegraphics[width=\linewidth]{figures/fai_g1.pdf}
        \caption{Gradient norm of group $\#2: g_2$ (Local level - Failed).}
    \end{subfigure}
    \hfill
    \begin{subfigure}{0.45\linewidth}
        \centering
        \includegraphics[width=\linewidth]{figures/suc_g1.pdf}
        \caption{Gradient norm of group $\#2: g_2$ (Local level - Successful).}
    \end{subfigure}
    \hfill
    \begin{subfigure}{0.45\linewidth}
        \centering
        \includegraphics[width=\linewidth]{figures/fai_cosine_l.pdf}
        \caption{Cosine similarity between $g_1$ and $g_1$ (Local level - Failed).}
    \end{subfigure}
    \hfill
    \begin{subfigure}{0.45\linewidth}
        \centering
        \includegraphics[width=\linewidth]{figures/suc_cosine_l.pdf}
        \caption{Cosine similarity between $g_1$ and $g_2$ (Local level - Successful).}
    \end{subfigure}
    \caption{Gradient visualization of a \textbf{failed} run (left) and a \textbf{successful} run (right) at local level. Result is received from finetuning \texttt{Roberta-large} on RTE on 2 GPUs. $g_1$ is received from GPU$\#1$ and $g_2$ is received from GPU$\#2$. 
    % \textcolor{red}{I feel the subcaptions are mismatched.}
    }
    \label{local_level_grad}
\end{figure*}

\begin{figure*}[t]
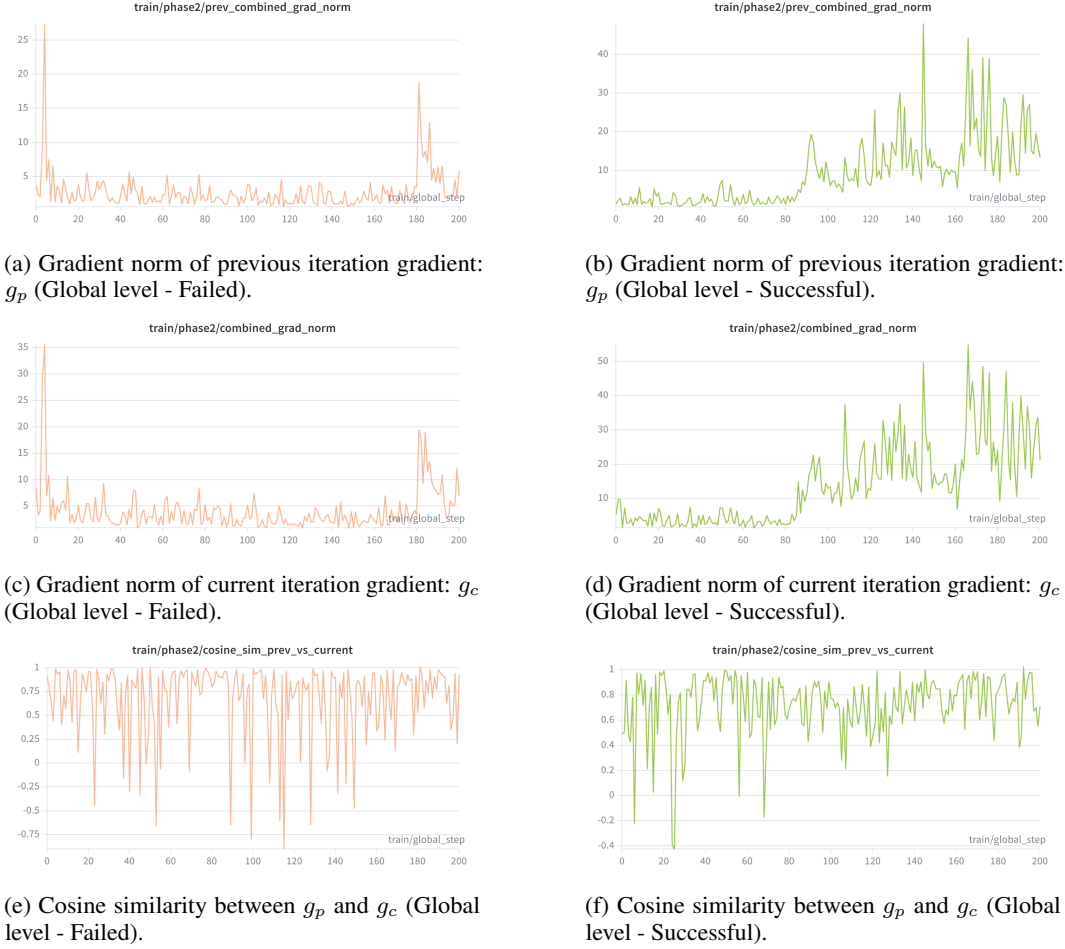

    \begin{subfigure}{0.45\linewidth}
        \centering
        \includegraphics[width=\linewidth]{figures/fai_combined_grad.pdf}
        \caption{Gradient norm of previous iteration gradient: $g_p$ (Global level - Failed).}
    \end{subfigure}
    \hfill
    \begin{subfigure}{0.45\linewidth}
        \centering
        \includegraphics[width=\linewidth]{figures/suc_combined_grad.pdf}
        \caption{Gradient norm of previous iteration gradient: $g_p$ (Global level - Successful).}
    \end{subfigure}
    \hfill
    \begin{subfigure}{0.45\linewidth}
        \centering
        \includegraphics[width=\linewidth]{figures/fai_combined_grad2.pdf}
        \caption{Gradient norm of current iteration gradient: $g_c$ (Global level - Failed).}
    \end{subfigure}
    \hfill
    \begin{subfigure}{0.45\linewidth}
        \centering
        \includegraphics[width=\linewidth]{figures/suc_combined_grad2.pdf}
        \caption{Gradient norm of current iteration gradient: $g_c$ (Global level - Successful).}
    \end{subfigure}
    \hfill
    \begin{subfigure}{0.45\linewidth}
        \centering
        \includegraphics[width=\linewidth]{figures/fai_cosine_g.pdf}
        \caption{Cosine similarity between $g_p$ and $g_c$ (Global level - Failed).}
    \end{subfigure}
    \hfill
    \begin{subfigure}{0.45\linewidth}
        \centering
        \includegraphics[width=\linewidth]{figures/suc_cosine_g.pdf}
        \caption{Cosine similarity between $g_p$ and $g_c$ (Global level - Successful).}
    \end{subfigure}
    \caption{Gradient visualization of a \textbf{failed} run (left) and a \textbf{successful} run (right) at global level. Result is received from finetuning \texttt{Roberta-large} on RTE on 2 GPUs. $g_1$ is received from GPU$\#1$ and $g_2$ is received from GPU$\#2$. Global gradient $g_p$ and $g_c$ are averaged over 2 devices.}
    \label{global_level_grad}
\end{figure*}

\begin{table*}[h]
\begin{tabular}{lcccc}
\toprule
              & \multicolumn{4}{c}{\textit{PubmedQA / GSM8K}}                                                   \\
              \cmidrule(lr){2-5}

& \multicolumn{1}{l}{Qwen2-7B} & Qwen2.5-14B & Mistral-7B-v0.3 & Ministral-14B-Base-2512 \\
\midrule
 & \multicolumn{4}{c}{LoRA settings}\\
\midrule
Rank          & 64&64&64&                                                                 64\\
Alpha         & 16&16&16&                                                                 16\\
Dropout       & 0.05&0.05&0.05&                                                               0.05\\
Modules       & \multicolumn{4}{c}{\{ q\_proj, k\_proj, v\_proj, o\_proj, gate\_proj, up\_proj, down\_proj \}}   \\
\midrule
 & \multicolumn{4}{c}{Others settings}\\
\midrule
Learning rate & \multicolumn{4}{c}{2E-4}                                                               \\
Batch size    & 8 & 4 & 8 & 4                                                              \\
Epoch         & 8 / 3&8 / 3&8 / 3&                                                                  8 / 3\\
LR schedule   & Cosine&Cosine&Cosine&                                                             Cosine\\
Warmup ratio  & 0.03&0.03&0.03&                                                               0.03\\
Max new tokens  & 256/512&256/512&256/512&                                                               256/512\\

\bottomrule
\end{tabular}
\caption{Hyperparameter settings for generative tasks.}
\label{tab:gen_settings}
\end{table*}

\begin{table}[h]
\centering
\begin{adjustbox}{max width=1\textwidth}
\begin{tabular}{lcccccccccccccc}
\toprule
                        & \multicolumn{2}{c}{PubmedQA} & \multicolumn{2}{c}{GSM8K} & \multicolumn{2}{c}{SuperGLUE} & \multicolumn{2}{c}{GLUE} & \multicolumn{2}{c}{CIFAR-10} & \multicolumn{2}{c}{CIFAR-100} & \multicolumn{2}{l}{Diabetes} \\
                        & m             & k            & m           & k           & m             & k             & m           & k          & m             & k            & m             & k             & m             & k            \\
\midrule
\texttt{Qwen2-7B}                &               4&              2&             4&             2&               $-$&               $-$&             $-$&            $-$&               $-$&              $-$&               $-$&               $-$&               $-$&              $-$\\
\texttt{Qwen2.5-14B}             &               4&              4&             4&             4&               $-$&               $-$&             $-$&            $-$&               $-$&              $-$&               $-$&               $-$&               $-$&              $-$\\
\texttt{Mistral-7B-v0.3}         &               4&              2&             2&             1&               $-$&               $-$&             $-$&            $-$&               $-$&              $-$&               $-$&               $-$&               $-$&              $-$\\
\texttt{Ministral-14B-Base-2512} &               4&              4&             4&             4&               $-$&               $-$&             $-$&            $-$&               $-$&              $-$&               $-$&               $-$&               $-$&              $-$\\
 \texttt{Llama-3.2-1B} & $-$& $-$& $-$& $-$& 2& 2& 2& 2& $-$& $-$& $-$& $-$& $-$&$-$\\
 \texttt{RoBERTa-large}& $-$& $-$& $-$& $-$& 2& 2& 2& 2& $-$& $-$& $-$& $-$& $-$&$-$\\
\texttt{ViT-Base}                &               $-$&              $-$&             $-$&             $-$&               $-$&               $-$&             $-$&            $-$&               2&              2&               2&               2&               $-$&              $-$\\
\texttt{Resnet-18}               &               $-$&              $-$&             $-$&             $-$&               $-$&               $-$&             $-$&            $-$&               2&              1&               2&               1&               $-$&              $-$\\
\texttt{Resnet-32}               &               $-$&              $-$&             $-$&             $-$&               $-$&               $-$&             $-$&            $-$&               2&              1&               2&               1&               $-$&              $-$\\
 \texttt{Resnet-32\_nr}& $-$& $-$& $-$& $-$& $-$& $-$& $-$& $-$& 2& 1& $-$& $-$& $-$&$-$\\
 \texttt{Resnet-56\_nr}& $-$& $-$& $-$& $-$& $-$& $-$& $-$& $-$& 2& 1& $-$& $-$& $-$&$-$\\
\texttt{MLP} \texttt{[32,32]}&               $-$&              $-$&             $-$&             $-$&               $-$&               $-$&             $-$&            $-$&               $-$&              $-$&               $-$&               $-$&               2&              1\\
\bottomrule
\end{tabular}
\end{adjustbox}
\caption{Detailed grouping settings of \mName for each task. $m$ is the number of intra-batch groups; $k$ is the number of inter-batch consecutive batches.}
\label{tab:grain_settings}
\end{table}

\begin{table*}[h]
\centering
\begin{adjustbox}{max width=\textwidth}
\begin{tabular}{llccccccc}
\toprule
& & \multicolumn{4}{c}{\textit{Sequence (Super)GLUE}} & \multicolumn{2}{c}{\textit{Imbalanced CIFAR}} \\
\cmidrule(lr){3-6}\cmidrule(lr){7-8}
& & RTE & MRPC & BoolQ & STS-B & CIFAR-10 & CIFAR-100 \\
\midrule
\multicolumn{2}{l}{Task type}                & cls.\ & cls.\ & cls.\ & reg.\ & cls.\ & cls.\ \\
\multicolumn{2}{l}{Number of classes}        & 2 & 2 & 2 & -- & 10 & 100 \\
\midrule
\multirow{4}{*}{Training samples}
& Long-tailed (100:1) & --   & --   & --   & -- & 12{,}406 & 10{,}847 \\
& Long-tailed (50:1)  & --   & --   & --   & -- & 36{,}223 & 36{,}029 \\
& Step (100:1)        & --   & --   & --   & -- & 25{,}250 & 25{,}250 \\
& Step (50:1)         & --   & --   & --   & -- & 37{,}500 & 37{,}500 \\
\multicolumn{2}{l}{Training samples (no imbalance)} & 2{,}490 & 3{,}668 & 9{,}427 & 5{,}749  & -- & -- \\
\midrule
\multicolumn{2}{l}{Validation samples}       & 277 & 408 & 3{,}270 & 1{,}500  & 10{,}000 & 10{,}000 \\
\bottomrule
\end{tabular}
\end{adjustbox}
\caption{Dataset statistics. ``cls.'' = classification, ``reg.'' = regression. STS-B values to be filled in. CIFAR train sizes are reported per imbalance pattern (long-tailed and step) at imbalance ratios 100:1 and 50:1 following~\citep{cao2019learning}; the validation set is held out unchanged.}
\label{tab:data_stats}
\end{table*}

\begin{table*}[h]
\centering
\begin{adjustbox}{max width=\textwidth}
\begin{tabular}{lcccccc}
\toprule
& \multicolumn{4}{c}{\textit{RoBERTa-large / Llama-3.2-1B}} & \multicolumn{2}{c}{\textit{ViT-base}} \\
\cmidrule(lr){2-5}\cmidrule(lr){6-7}
                        & RTE     & MRPC    & BoolQ   & STS-B & CIFAR-10 & CIFAR-100 \\
\midrule
Batch size per device& 5       & 5       & 5       &       5& 64& 64\\
Gradient accumulation   & 2       & 2       & 2       &       2& 2        & 2 \\
Learning rate           & \multicolumn{4}{c}{\{5E-6, 1E-5, 2E-5, 3E-5, 1E-4, 2E-4, 3E-4\}} & \multicolumn{2}{c}{\{1E-1, 1E-2, 1E-3\}} \\
LR scheduler            & linear  & linear  & linear  & linear & linear       & linear \\
Max sequence length     & 512     & 256     & 512     &       256& --       & -- \\
Epochs                  & 10      & 3       & 10      &       3& 5        & 5 \\
Number of GPUs          & 2       & 2       & 2       & 2     & 2        & 2 \\
\bottomrule
\end{tabular}
\end{adjustbox}
\caption{Hyperparameter settings for sequence classification/regression (RoBERTa-large and Llama-3.2-1B) and image classification (ViT-base). Values in the LM block apply to both backbones; the ViT block applies to all imbalanced CIFAR variants.}
\label{tab:hp_setting}
\end{table*}